\documentclass{article}

\usepackage{microtype}
\usepackage{graphicx}
\usepackage{caption,subcaption}
\usepackage{booktabs} 
\usepackage{hyperref}

\usepackage[accepted]{icml2023}

\usepackage{amsmath}
\usepackage{amssymb}
\usepackage{mathtools}
\usepackage{amsthm}
\usepackage{bbm}
\usepackage{xspace}
\usepackage[shortlabels,inline]{enumitem}

\usepackage[linesnumbered,ruled,vlined]{algorithm2e}
\SetKw{KwInit}{Initialize}
\SetKwProg{KwUpdate}{determine}{}{}
\SetKwProg{KwSkip}{skipping}{}{}
\definecolor{DarkRed}{rgb}{0.5, 0.1, 0.1}
\definecolor{DarkGreen}{rgb}{0, 0.5, 0}
\definecolor{DarkCyan}{rgb}{0, 0.5, 0.5}

\SetCommentSty{mycommfont}

\usepackage[capitalize,noabbrev]{cleveref}

\usepackage{thmtools,thm-restate}

\usepackage{tikz}

\theoremstyle{plain}
\newtheorem{theorem}{Theorem}[section]

\newtheorem{lemma}[theorem]{Lemma}
\newtheorem{corollary}[theorem]{Corollary}

\theoremstyle{definition}

\theoremstyle{remark}
\newtheorem{remark}[theorem]{Remark}

\renewcommand{\hat}{\widehat}
\renewcommand{\tilde}{\widetilde}
\newcommand{\e}{\varepsilon}

\DeclarePairedDelimiter{\lr}{(}{)}
\DeclarePairedDelimiter{\lrs}{[}{]}
\DeclarePairedDelimiter{\lrc}{\{}{\}}
\DeclarePairedDelimiter{\abs}{|}{|}
\DeclarePairedDelimiter{\floor}{\lfloor}{\rfloor}

\newcommand{\N}{\mathbb{N}}

\newcommand{\bbE}{\mathbb{E}}
\newcommand{\E}[1]{\bbE\lrs*{#1}}
\newcommand{\bbP}{\mathbb{P}}
\newcommand{\pr}[1]{\bbP\lr*{#1}}
\newcommand{\bigO}{\mathcal{O}}
\newcommand{\DKL}{D_\mathrm{KL}}
\newcommand{\KL}[1]{\DKL(#1)}
\newcommand{\dtv}{d_{\mathrm{TV}}}

\newcommand{\calA}{\mathcal{A}}
\newcommand{\calB}{\mathcal{B}}
\newcommand{\calF}{\mathcal{F}}
\newcommand{\calL}{\mathcal{L}}
\newcommand{\calR}{\mathcal{R}}
\newcommand{\calS}{\mathcal{S}}
\newcommand{\calT}{\mathcal{T}}
\newcommand{\calM}{\mathcal{M}}
\newcommand{\calD}{\mathcal{D}}
\newcommand{\calC}{\mathcal{C}}
\newcommand{\calJ}{\mathcal{J}}
\newcommand{\calE}{\mathcal{E}}

\newcommand{\frakD}{\mathfrak{D}}

\DeclareMathOperator*{\argmin}{arg\,min}
\DeclareMathOperator*{\argmax}{arg\,max}

\newcommand{\setting}{BIO\xspace}
\newcommand{\nsd}{\textrm{\textnormal{\textup{\texttt{NSD-UCRL2}}}}\xspace}
\newcommand{\metaif}{\textrm{\textnormal{\textup{\texttt{Meta\setting}}}}\xspace}
\newcommand{\metaifswitch}{\textrm{\textnormal{\textup{\texttt{MetaAda\setting}}}}\xspace}
\newcommand{\dadaexp}{\textrm{\textnormal{\textup{\texttt{DAda-Exp3}}}}\xspace}
\newcommand{\ucb}{\textrm{\textnormal{\textup{\texttt{UCB1}}}}\xspace}

\newcommand{\expalg}{\textrm{\textnormal{\textup{\texttt{Exp3}}}}\xspace}

\newcommand{\tsallisinf}{\textrm{\textnormal{\textup{\texttt{Tsallis-INF}}}}\xspace}

\newcommand{\TODO}[1]{%
\ifmmode
\text{\textcolor{red}{TODO: #1}}
\else
\textcolor{red}{TODO: #1}
\fi
}

\usepackage[textsize=tiny]{todonotes}

\icmltitlerunning{Delayed Bandits: When Do Intermediate Observations Help?}

\begin{document}

\twocolumn[
\icmltitle{Delayed Bandits: When Do Intermediate Observations Help?}



\icmlsetsymbol{equal}{*}

\begin{icmlauthorlist}
\icmlauthor{Emmanuel Esposito}{equal,unimi,iit}
\icmlauthor{Saeed Masoudian}{equal,ku}
\icmlauthor{Hao Qiu}{unimi}
\icmlauthor{Dirk van der Hoeven}{uva}

\icmlauthor{Nicolò Cesa-Bianchi}{unimi,polimi}
\icmlauthor{Yevgeny Seldin}{ku}
\end{icmlauthorlist}

\icmlaffiliation{unimi}{Università degli Studi di Milano, Milan, Italy}
\icmlaffiliation{iit}{Istituto Italiano di Tecnologia, Genoa, Italy}
\icmlaffiliation{polimi}{Politecnico di Milano, Milan, Italy}
\icmlaffiliation{ku}{University of Copenhagen, Copenhagen, Denmark}
\icmlaffiliation{uva}{Korteweg-de Vries Institute for Mathematics University of Amsterdam, Amsterdam, Netherlands}

\icmlcorrespondingauthor{Emmanuel Esposito}{\href{mailto:emmanuel@emmanuelesposito.it}{\texttt{emmanuel@emmanuelesposito.it}}}
\icmlcorrespondingauthor{Saeed Masoudian}{\href{mailto:saeed.masoudian@di.ku.dk}{\texttt{saeed.masoudian@di.ku.dk}}}

\icmlkeywords{Machine Learning, ICML}

\vskip 0.3in
]



\printAffiliationsAndNotice{\icmlEqualContribution} 

\begin{abstract}
We study a $K$-armed bandit with delayed feedback and intermediate observations. We consider a model where intermediate observations have a form of a finite state, which is observed immediately after taking an action, whereas the loss is observed after an adversarially chosen delay.
We show that the regime of the mapping of states to losses determines the complexity of the problem, irrespective of whether the mapping of actions to states is stochastic or adversarial.
If the mapping of states to losses is adversarial, then the regret rate is of order $\sqrt{(K+d)T}$ (within log factors), where $T$ is the time horizon and $d$ is a fixed delay. This matches the regret rate of a $K$-armed bandit with delayed feedback and without intermediate observations, implying that intermediate observations are not helpful. However, if the mapping of states to losses is stochastic, we show that the regret grows at a rate of $\sqrt{\big(K+\min\{|\mathcal{S}|,d\}\big)T}$ (within log factors), implying that if the number $|\mathcal{S}|$ of states is smaller than the delay, then intermediate observations help. We also provide refined high-probability regret upper bounds for non-uniform delays, together with experimental validation of our algorithms.
\end{abstract}

\section{Introduction}

\emph{Delay} is an ubiquitous phenomenon that many sequential decision makers have to deal with. For example, outcomes of medical treatments are often observed with delay, purchase events happen with delay after advertisement impressions, and acceptance/rejection decisions for scientific papers are observed with delay after manuscript submissions. The impact of delay on the performance of sequential decision makers, measured by regret, has been extensively studied under full information and bandit feedback, and in stochastic and adversarial environments.
Yet, in many situations in real life \emph{intermediate observations} may be available to the learner. For example, a health check-up might give a preliminary indication on the effect of a treatment, an advertisement click might be a precursor for an upcoming purchase, and preliminary reviews might provide some information regarding an upcoming acceptance or rejection decision. In this work we study when, and how, intermediate observations can be used to reduce the impact of delay in observing the final outcome of an action in a multi-armed bandit setting.

Online learning with delayed feedback and intermediate observations was studied by \citet{mann2019learning} in a full-information setting, and then by \citet{Vernade0M20} in a nonstationary stochastic bandit setting. In the paper of \citet{Vernade0M20}, at each time step the learner chooses an action and immediately observes a signal (also called state) belonging to a finite set. The actual loss (i.e., feedback) incurred by the learner in that time step is only received with delay, which can be fixed or random. More formally, the observed state is drawn from a distribution that only depends on the chosen action, and the incurred loss is drawn from a distribution that only depends on the observed state (and not on the chosen action), forming a Markov chain.
    \begin{tikzpicture}[mynode/.style={inner sep=10pt,align=center,font=\large},scale=0.935]
        \node[mynode] (action) at (0,0) {Action\\{\normalsize $A_t$}};
        \node[mynode] (state)  at (3.5,0) {State\\{\normalsize $S_t=s_t(A_t)$}};
        \node[mynode] (loss)   at (7,0) {Loss\\{\normalsize $\ell_t(S_t)$}};
    
        \path[->,line width=1] (action) edge node[above] {no delay}    (state);
        \path[->,line width=1] (state)  edge node[above] {delay $d_t$} (loss);
    \end{tikzpicture}
The work of \citet{Vernade0M20} studies a setting, where $s_t$ are nonstationary and $\ell_t$ are i.i.d.\ stochastic.

In this work, we consider two possible regimes for the mappings $s_t$ from actions to states (stochastic and adversarial) and two possible regimes for the mappings $\ell_t$ from states to losses (also stochastic and adversarial). Altogether, we study four different regimes, defined by the combination of the first and the second mapping type.

We characterize (within logarithmic factors) the minimax regret rates for all of them, by giving upper and lower bounds. Similar to \citeauthor{Vernade0M20}, we assume that the states are observed instantaneously, and we assume that the losses are observed with delay $d$. We show that the minimax regret rate is fully determined by the regime of the states to losses mapping, regardless of the regime of the actions to states mapping. The results are informally summarized in the following table, where $K$ denotes the number of actions, $S$ denotes the number of states, and $T$ denotes the time horizon. It is assumed that the losses belong to the $[0,1]$ interval. 
\begin{center}
\begin{tabular}{|c|c|}
\hline
States to losses mapping & Regret (within log factors) \\ \hline\hline
\raisebox{-1mm}{Adversarial}  &   \raisebox{-1mm}[3mm][3mm]{$\sqrt{(K+d)T}$}     \\ \hline
\raisebox{-2mm}{Stochastic} &   \raisebox{-2mm}[3mm][5mm]{$\sqrt{\big(K+\min\{S,d\}\big)T}$}   \\ \hline
\end{tabular}
\end{center}
All of our upper bounds hold with high probability (with respect to the learner's internal randomization) irrespective of the regime of the action to states mapping.

We recall that (within logarithmic factors) the minimax regret rate in multi-armed bandits with delays without intermediate observations is of order $\sqrt{(K+d)T}$ \citep{cesa2019delay}. Therefore, we conclude that if the mapping from states to actions is adversarial, then intermediate observations do not help (in the minimax sense), because the regret rates are the same irrespective of whether the intermediate observations are used or not, and irrespective of whether the mapping from actions to states is stochastic or adversarial. However, if the mapping from states to losses is stochastic, and the number $S$ of states is smaller than the delay $d$, then intermediate observations are helpful, and we provide an algorithm, \metaifswitch, which is able to exploit them. Our result improves on the $\widetilde{\mathcal{O}}\big(\!\sqrt{KST}\big)$ regret bound obtained by \citet{Vernade0M20} for the case of stochastic and stationary action to states mapping.
Our algorithm also applies to a more general setting of non-uniform delays $\lr{d_t}_{t \in [T]}$ where we achieve a high-probability regret bound of order $\sqrt{KT + \min\lrc{ST, \calD_T}}$ (ignoring logarithmic factors). This improves upon the total delay term $\calD_T = d_1+\cdots+d_T$ similarly to the respective term in the fixed delay setting.

\textbf{Related work}
Adaptive clinical trials have served an inspiration for the multi-armed bandit model \citep{Tho33}, and, interestingly, they have also pushed the field to study the effect of delayed feedback \citep{Sim77,Eic88}.
In the bandit setting \citet{joulani2013online} have studied a stochastic setting with random delays, whereas \citet{neu2010online, neu2014online} have studied a nonstochastic setting with constant delays. \citet{cesa2019delay} have shown an $\Omega(\max\{\sqrt{KT}, \sqrt{d T\ln K}\})$ lower bound for nonstochastic bandits with uniformly delayed feedback, and an upper bound matching the lower bound within logarithmic factors by using an \textsc{Exp3}-style algorithm \citep{auer2002nonstochastic}, whereas \citet{zimmert20} have reduced the gap to the lower bound down to constants by using a Tsallis-INF approach \citep{ZimmertS21}. Follow up works have studied adversarial multi-armed bandits with non-uniform delays \citep{thune2019nonstochastic, bistritz2019exp3, bistritz2021nodiscounted, gyorgy21, van2022nonstochastic} with \citet{zimmert20} providing a minimax optimal algorithm and \citet{Masoudian2022} deriving a matching lower bound and a best-of-both-worlds extension. Two key techniques for handling non-uniform delays are skipping, introduced by \citet{thune2019nonstochastic}, and algorithm parametrization by the number of outstanding observations (an observed quantity at action time), as opposed to the delays (an unobserved quantity at action time), introduced by \citet{zimmert20}.

\textbf{Paper structure}
In \Cref{sec:def} we provide a formal problem definition. In \Cref{sec:algo} we introduce two algorithms, \metaif and \metaifswitch, for the model of bandits with intermediate observations. In \Cref{sec:analysis} we analyze both algorithms and prove high-probability regret bounds for the setting of adversarial action-state mappings and stochastic losses. In \Cref{sec:lower} we provide the lower bounds, and in \Cref{sec:experiments} experimental evaluation, concluding with a discussion in \Cref{sec:discussion}.

\section{Problem definition}
\label{sec:def}
We consider an online learning setting with a finite set $\calA = [K]$ of $K \ge 2$ actions and a finite set $\calS = [S]$ of $S \ge 2$ states.
In each round $t=1,2,\ldots$ the learner picks an action $A_t\in\calA$ and receives a state $S_t = s_t(A_t) \in \calS$ as an intermediate observation according to some mapping $s_t \in \calS^{\calA}$.
The learner also incurs a loss $\ell_t(S_t) \in [0,1]$, which is only observed at the end of round $t+d_t$, where the delay $d_t \ge 0$ is revealed to the learner only when the observation is received.

The difficulty of this learning task depends on three elements all initially unknown to the learner:
\begin{itemize}[topsep=0pt,itemsep=0pt,leftmargin=9pt]
    \item the sequence of action-state mappings $s_1, \dots, s_T \in \calS^{\calA}$;
    \item the sequence of loss vectors $\ell_1, \dots, \ell_T \in [0,1]^S$;
    \item the sequence of delays $d_1, \dots, d_T \in \N$, where $d_t \le T-t$ for all $t \in [T]$ without loss of generality.
\end{itemize}
Note that unlike standard bandits, here the losses are functions of the states instead of the actions.
However, since actions are chosen without a-priori information on the action-state mappings, learners have no direct control on the losses they will incur and, because of the delays, they also have no immediate feedback on the loss associated with the observed states.
Note also that, for all $t \ge 1$, the states $s_t(a)$ for $a\neq A_t$ and the losses $\ell_t(s)$ for $s\neq S_t$ are never revealed to the algorithm.
For brevity, we refer to this setting as (delayed) Bandits with Intermediate Observations (\setting).

In the setting of stochastic losses, we assume the loss vectors $\ell_t \in [0,1]^S$ are sampled i.i.d.\ from some fixed but unknown distribution $Q$, and let $\theta \in [0,1]^\calS$ be the unknown vector of expected losses for the states. That is, $\ell_t(s) \sim Q(\cdot \,|\, s)$ has mean $\theta(s)$ for each $t \in [T]$ and $s \in \calS$. Note that we allow dependencies between the stochastic losses of distinct states in the same round, but require losses to be independent across rounds. In the setting of stochastic action-state mappings, we assume that each observed state $S_t$ is independently drawn from a fixed but unknown distribution $P(\cdot \,|\, A_t)$. If both losses and action-state mappings are stochastic, then $\ell_t(S_t)$ is independent of $A_t$ given $S_t$.
When losses or action-state mappings are adversarial, we always assume oblivious adversaries.

Our main quantity of interest is the regret measured via the learner's cumulative loss $\sum_t \ell_t(S_t)$, where $S_t = s_t(A_t)$ and $(A_t)_{t\ge 1}$ is the sequence of learner's actions. In case of stochastic losses, we define the learner's performance by $\sum_t \theta(S_t)$. In case of stochastic action-state mappings, we average each instantaneous loss over the random choice of the state: $\sum_s \ell_t(s) P(s \,|\, A_t)$ for adversarial losses and $\sum_s \theta(s) P(s \,|\, A_t)$ for stochastic losses. Regret is always computed according to the best action with respect to appropriate notion of cumulative loss. In particular, for stochastic state-action mappings, the cumulative losses of the best action are
\[
   \min_{a\in\calA} \sum_{t=1}^T \sum_{s\in\calS} \ell_t(s) P(s \,|\, a)
\]
and
\[
\min_{a \in \calA} \sum_{t=1}^T \sum_{s\in\calS} \theta(s) P(s \,|\, a) \enspace.
\]

\section{Algorithm}
\label{sec:algo}
In this section we introduce \metaif (\Cref{alg:metaif}) that transforms any algorithm $\calB$ tailored for the delayed setting \emph{without} intermediate observations into an algorithm for our setting. We then propose \metaifswitch, a modification of \metaif that delivers an improved regret bound for our setting. 

The idea of \metaif is to reduce the impact of delays using the information we get from intermediate observations. More precisely, if we have \emph{enough} observations for the current state $S_t$ at time $t$, we immediately feed to $\calB$ the \emph{estimate} of the mean loss of this state as if it were the actual loss at time $t$; otherwise, we wait for $d_t$ time steps and refine our estimate using the additional loss observations.

The are two key steps in the design of our algorithm: \emph{how} we construct the mean estimate and \emph{when} we use it instead of waiting for the actual loss.
They are the steps highlighted in green in \Cref{alg:metaif}.
For all $t \in [T]$ and $s \in \calS$, we use $\tilde\theta_t(s)$ to denote the mean estimate of $\theta(s)$ at round $t$ and $n_t(s)$ to denote the number of observations for state $s$ that we want to observe before using $\tilde\theta_t(s)$. We add a subscript $t$ to $\calL(s)$ in \Cref{alg:metaif} to denote the set of observations we have collected at the end of round $t$. Thus, $\tilde\theta_t(s)$ uses $N_t(s) = |\calL_t(s)|$ observations.

\begin{algorithm}
\caption{\metaif}
\label{alg:metaif}
\DontPrintSemicolon
\LinesNumberedHidden
\KwIn{Algorithm $\calB$ for standard delayed bandits, confidence parameter $\delta \in (0,1)$}
\KwInit{$\calL(s) = \emptyset$ for all $s \in \calS$}\;
\For{$t = 1,\ldots,T$}{
    Get $A_t$ from $\calB$\;
    Observe $S_t = s_t(A_t)$\;
    \For{$j: j+d_j = t$}{
        Receive $(j,\ell_j(S_j))$\;
        Update $\calL(S_j) = \calL(S_j) \cup \lrc{(j,\ell_j(S_j))}$\;
    }
    Initialize feedback set $\calM = \emptyset$\;
    {\color{DarkGreen}Compute $n_t(S_t)$\;}
    \If {$|\calL(S_t)| \geq n_t(S_t)$}{
        Add $t$ to $\calM$ \;
    }
    \For{$j: j+d_j = t \wedge |\calL(S_j)| < n_j(S_j)$ }{
        Add $j$ to $\calM$ \;
    }
    \For{$j \in \calM$}{
        {\color{DarkGreen}Compute $\tilde\theta_t(S_j)$ from $\calL(S_j)$}\tcp*{using $\delta$}
        Feed $\lr{j, A_j, \tilde\theta_t(S_j)}$ to $\calB$\; 
    }
}
\end{algorithm}

\textbf{Fixed delay setting.}
When all rounds have delay $d$, we simply choose $n_t(s) = d$ for all $s \in \calS, t \in [T]$. In other words, if we have at least $d$ observations for some state, then we can compensate for the effect of delays and construct a well concentrated mean estimate around the actual mean. Let $\hat\theta_{t}(s) = \sum_{j \in \calL_t(s)} \ell_{j}(s)\big/N_t(s)$. Then our mean loss estimate is a lower confidence bound for $\theta(s)$ defined by
\begin{equation}\label{eq:lowerconf-estimates}
    \Tilde{\theta}_{t}(s) = \max\lrc*{0, \hat\theta_t(s) - \frac12 \e_t(s)}
\end{equation}
for $\e_t(s) = \sqrt{\frac{2}{N_t(s)} \ln\frac{4ST}{\delta}}\,$.

\textbf{Arbitrary delay setting.}
In the arbitrary delay setting, where we do not have preliminary knowledge of delays, we can not use the delays to set $n_t(s)$. Instead, at the \emph{end} of time $t$, we have access to the number of outstanding observations $\sigma_t = \big|\lrc{j \in [t] \,:\, j + d_j > t}\big|$, which is different from prior works that consider outstanding observation at the \emph{beginning} of the round. Then, for any $s \in \calS$, we may set
$
n_t(s) = \sigma_t
$.
With this choice, incurring zero delay at some round implies that we received at least half of all the observations we could have received in the no-delay setting (see \cref{sec:metaiferrorterms}).
In \Cref{sec:analysis} we see that this ensures our mean estimate is well concentrated around its mean. 

Since \Cref{alg:metaif} waits for the actual loss at time $t$ only if $N_t(S_t) < \sigma_t$, then $\Tilde{d}_t = d_t\,\mathbbm{1}[N_t(S_t) < \sigma_t]$ is the actual delay incurred by the algorithm, and $\calL_{t+\Tilde{d}_t}(s)$ is the set of observations used to compute the estimate of the mean loss at time $t$.
Because some observations may arrive at the same time, the high-probability analysis of \metaif requires these observations to be ordered. More precisely, we construct our mean estimate at time $t+\tilde{d}_t$ for the feedback of round $t$ using the set
\[
\calL_t'(s) = \Bigl\{(j, \ell_j(s)) \in \calL_{t+\tilde{d}_t}(s) \,\Big|\, j + \tilde{d}_j \!= t + \Tilde{d}_t \Rightarrow j \!<\! t \Bigr\}.
\]
Letting $N_t'(s) = |\calL_t'(s)|$, we define the empirical mean
\begin{equation} \label{eq:unbiased-estimates}
\hat\theta_{t}(s) = \sum_{j \in \calL_t'(s)}\frac{\ell_{j}(s)}{N_t'(s)} \enspace.
\end{equation}
Then, we set $\e_t(s) = \sqrt{\frac{2}{N'_t(s)} \ln\frac{4ST}{\delta}}$ and define the mean loss estimator similarly to \Cref{eq:lowerconf-estimates}.

\textbf{The \metaifswitch algorithm.}
As we said already, the goal of intermediate observations is to reduce the impact of delays. However, if the number of states is too large compared to the average delay, then the information we get from intermediate observations could be misleading. To address this issue, we introduce \metaifswitch (\Cref{alg:metaswtich}). Given a horizon $T$,\footnote{Note that we may remove the a-priori knowledge of $T$ by using a doubling trick at the cost of a polylog factor in the regret. See \Cref{remark:a-priori-knowledge} for further details.} this algorithm runs $\calB$ (which is tailored for the setting \emph{without} intermediate observations) until the total incurred delay exceeds $S T$, and then switches to \metaif.
We precise that \metaifswitch computes $\frakD_t$ as the sum of outstanding observation counts up to round $t$, which is then used in the switching condition.

\begin{algorithm}
\caption{\metaifswitch}
\label{alg:metaswtich}
\DontPrintSemicolon
\LinesNumberedHidden
\KwIn{Algorithm $\calB$ for standard delayed bandits, confidence parameter $\delta \in (0,1)$}
\KwInit{$\frakD_0 = 0$}\;
\For{$t = 1,\ldots,T$}{
    Get $A_t$ from $\calB$\;
    \For{$j: j+d_j = t$}{
        Receive $(j,\ell_j(S_j))$\;
        Feed $(j,A_j, \ell_j(S_j))$ to $\calB$\;
    }
    Set $\sigma_t = \sum_{j=1}^{t-1} \mathbbm{1}[j + d_j > t]$\;
    Update $\frakD_t = \frakD_{t-1} + \sigma_t$ \;
    \If{$\frakD_t \lr{3\ln K + \ln(6/\delta)} > 49 ST \ln \frac{8ST}{\delta}$}{
    \textbf{break}\;
    }
}
\If{$t < T$}{
    Run $\metaif(\calB, \delta/2)$ for the remaining rounds\;
}
\end{algorithm}

\section{Regret Analysis}
\label{sec:analysis}
We analyze \metaif and \metaifswitch in the setting of adversarial action-state mappings and stochastic losses where the regret is defined by
\[
    R_T = \sum_{t=1}^T \theta(S_t) - \min_{a \in \calA} \sum_{t=1}^T \theta(s_t(a)) \enspace.
\]
Our analysis guarantees a bound on $R_T$ that holds with high probability (and not just in expectation).
A related notion of regret is
\[
    \calR_T = \sum_{t=1}^T \ell_t(S_t) - \min_{a \in \calA} \sum_{t=1}^T \ell_t(s_t(a))
\]
which considers the realized losses instead of their means.
The two quantities are close with high probability: each inequality
\begin{align}\label{eq:regretrelations}
    -\sqrt{2T \ln(2K/\delta)} \le R_T - \calR_T \le \sqrt{2T \ln(2/\delta)}
\end{align}
individually holds with probability at least $1-\delta$ for any given $\delta \in (0,1)$ (see \Cref{lem:regret-notions-comparison}).

Let $\calD_T = \sum_{t=1}^T d_t$ be the total delay.
We start by showing an upper bound on the total actual delay $\tilde\calD_T = \sum_{t=1}^T d_t \mathbbm{1}[N_t(S_t) < \sigma_t] \le \calD_{T}$ incurred by \metaif.
Then, we provide a high-probability regret analysis of both \metaif and \metaifswitch.

More precisely, we can show that \metaif incurs the delays of no more than $\min\lrc{2S\sigma_{\max}, T}$ rounds, where $\sigma_{\max} = \max_{t \in [T]} \sigma_t$.
In the worst case, these rounds correspond with those from the set
\begin{equation}\label{eq:phi}
    \Phi \in \argmax_{\calJ \subseteq [T]} \Big\{\calD_{\calJ} : \abs{\calJ} = \min\lrc{2S\sigma_{\max}, T}\Big\} \enspace.
\end{equation}
where we denote $\calD_{\calJ} = \sum_{t \in \calJ} d_t$ for any $\calJ \subseteq [T]$.
Note that the set $\Phi$ is fully determined by the delay sequence $d_1,\ldots,d_T$.
Moreover, the total delay incurred by \metaif cannot be worse than the sum of delays corresponding to the rounds in $\Phi$, as stated in the lemma below.
\begin{restatable}[Total actual delay]{relemma}{relemmatotaldelaybound}\label{lem:total-delay-bound}
    If \metaif is run with any algorithm $\calB$ on delays $(d_t)_{t \in [T]}$, then
    $
        \Tilde{\calD}_T \le \calD_{\Phi}
    $.
\end{restatable}
\Cref{lem:total-delay-bound} (proof in \Cref{app:delaybound})
implies that, if all delays are bounded by $d_{\max}$, then $\Tilde{\calD}_T \le 2S \sigma_{\max} d_{\max}$, which does not depend on $T$.
In the fixed-delay setting with delay $d$, for example, we get a total effective delay of at most $2S d^2$, rather than the total delay $dT$ we would incur without access to intermediate observations (when $T$ is large enough).

We now turn \metaif into a concrete algorithm by instantiating $\calB$. Specifically, we use \dadaexp \citep{gyorgy21}, a variant of \expalg which does not use intermediate observations and is robust to delays. \dadaexp has the following regret bound.
\begin{theorem}[{\citet[Corollary~4.2]{gyorgy21}}]\label{theorem:dadaexp}
    For any $\delta \in (0,1)$, the regret with respect to realized losses of \dadaexp in the adversarial bandits with arbitrary delays with probability at least $1-\delta$ satisfies 
    \begin{align*}
    \calR_T &\leq 2\sqrt{3\lr{2KT + \calD_T}\ln K}\\
    &\quad + \biggl(\sqrt{\frac{2KT + \calD_T}{3\ln K}} + \frac{\sigma_{\max}}2 + 1\biggr) \ln\frac{2}{\delta}\enspace.
    \end{align*}
\end{theorem}
While \Cref{theorem:dadaexp} shows a high-probability bound on $\calR_T$, \Cref{eq:regretrelations} shows that a high-probability bound for one notion of regret ensures a high-probability bound for the other. Although the original bound by \citet{gyorgy21} was stated with $d_{\max}$ instead of $\sigma_{\max}$, we can replace the former with the latter by observing that, in the analysis of \citet[Theorem~4.1]{gyorgy21}, they only use $d_{\max}$ to upper bound the number of outstanding observations.
Note that $\sigma_{\max}$ is never larger than $d_{\max}$, indicating it is a well-behaved term that is not vulnerable to a few large delays. See \citet[Lemma 3]{Masoudian2022} for a refined quantification of the relation between $\sigma_{\max}$ and $d_{\max}$.

If we consider a fixed confidence level $\delta \in (0,1)$, then we can make the learning rate $\eta_t$ and the implicit exploration term $\gamma_t$ in \dadaexp depend on the specific value of $\delta$ so as to achieve an improved regret bound (see Appendix~\ref{app:improved-dadaexp-bound}).
This allows us to show that in the \setting setting with adversarial action-state mappings and stochastic losses, the regret $\calR_T$ of \dadaexp is upper bounded by
\begin{align}
    2\sqrt{2KT C_{K,6\delta}}
    + 2\sqrt{D_T C_{K,6\delta}} + \frac{\sigma_{\max} + 2}{2}\ln\frac{2}{\delta} \label{eq:dadaexp3bound}
\end{align}
with probability at least $1-\delta$, where
\[
    C_{K,\delta} = 3\ln K + \ln\frac{12}{\delta} \enspace.
\]

Next, we state the regret bound for \metaif.
We remark that we initialize \dadaexp with confidence parameter $\delta/2$ so as to guarantee the high-probability bound as in \eqref{eq:dadaexp3bound} with probability at least $1-\delta/2$ as required.
\begin{restatable}{rethm}{upperboundmetaif}\label{theorem:metaif}
    Let $\delta \in (0,1)$. If we run \metaif using \dadaexp, then the regret of \metaif in the \setting setting with adversarial action-state mappings and stochastic losses with probability at least $1-\delta$ satisfies 
    \begin{align}
        R_T &\leq 2\sqrt{2KTC_{K,3\delta}} + 7\sqrt{ST \ln\frac{4ST}{\delta}} \notag \\
        &\quad + 2\sqrt{\calD_{\Phi} C_{K,3\delta}} + \frac{\sigma_{\max} + 2}{2}\ln\frac{4}{\delta}\enspace. \label{eq:metaifbound}
    \end{align}
\end{restatable}

We begin the analysis of \Cref{theorem:metaif} by decomposing the regret into two parts:
\begin{enumerate*}[(i)]
    \item the regret $\calR_T$ of \dadaexp with losses $\tilde\theta_t(S_t)$, and
    \item the gap $R_T - \calR_T$, corresponding to the cumulative error of the estimates fed to \dadaexp.
\end{enumerate*}
For the first part, we follow an approach similar to \citet{gyorgy21} and apply \citet[Lemma~1]{Neu15Implicit} to obtain a concentration bound for the loss estimates defined using importance weighting along with implicit exploration.
When using the actual losses, the application of \citet[Lemma~1]{Neu15Implicit} is straightforward.
However, when the mean loss estimate $\tilde\theta_t(S_t)$ is used rather than the actual loss, there is a potential dependency between the chosen action $A_t$ and $\tilde\theta_t(S_t)$.
In \Cref{app:reductiontodada} we carefully design a filtration to show that we may indeed use the high-probability regret bound of \dadaexp in order to upper bound the first part (regret $\calR_T$ defined in terms of the estimates $\tilde{\theta}_t$).

The second part requires to bound the cumulative error of our estimator in \eqref{eq:unbiased-estimates} for the observed states $\lrc{S_t}_{t \in [T]}$.
To this end, we use the Azuma-Hoeffding inequality to control the error of these estimates.
Doing so causes a $\tilde{\bigO}\lr{\sqrt{ST}}$ term to appear in the regret bound.
The detailed proof of this part is in \Cref{sec:metaiferrorterms}, together with the proof of \Cref{theorem:metaif}.

The presence of the $\tilde{\bigO}\lr{\sqrt{ST}}$ term in the regret bound implies that, when $S \gg \max\lrc{\calD_T/T, K}$, using intermediate feedback leads to no advantage over ignoring it.
So we ideally want to recover the original bound in \eqref{eq:dadaexp3bound} when this happens.
\metaifswitch solves this issue and gives the following regret guarantee.
The proof of this result is deferred to \cref{sec:metaifswitch-analysis}.
We remark that, to achieve this bound, before the eventual switch we use algorithm \dadaexp with confidence parameter set to $\delta/3$ so as to guarantee a high-probability bound on $R_{t^*}$ with probability at least $1-\delta/2$ over the first $t^*$ rounds that \dadaexp runs by itself.
\begin{restatable}{rethm}{upperboundmetaifswitch}\label{theorem:metaifswitch}
    Let $\delta \in (0,1)$. If we run \metaifswitch with \dadaexp, then the regret of \metaifswitch in the \setting setting with adversarial action-state mappings and stochastic losses with probability at least $1-\delta$ satisfies 
    \begin{align}\label{eq:metaifboundswitch}
        & R_T\leq 3\min\lrc*{ 7\sqrt{ST \ln\frac{8ST}{\delta}}, \sqrt{\calD_T C_{K,2\delta}} }  \\
        &\quad + 6\sqrt{KT C_{K,2\delta}} + 2\sqrt{\calD_{\Phi}  C_{K,2\delta}} + (\sigma_{\max} + 2)\ln\frac{8}{\delta}\enspace. \notag 
    \end{align}
\end{restatable}

If we consider any upper bound $d_{\max}$ on the delays $(d_t)_{t \in [T]}$, we can further observe that the regret $R_T$ of \metaifswitch (with \dadaexp) satisfies
\begin{align*}
    R_T = \tilde{\bigO}\lr*{\sqrt{KT} + \min\lrc*{\sqrt{S} \bigl(\sqrt{T} + d_{\max}\bigr), \sqrt{d_{\max}T}}}
\end{align*}
with high probability.
This also follows from the fact that, as previously mentioned, we can bound the total delay of \metaif by $\calD_{\Phi} \le 2Sd_{\max}^2$.

Given the previous regret bounds, we observe that we may further improve the dependency on the delays by adopting the idea of skipping rounds with large delays when computing the learning rates.
This ``skipping'' idea was introduced by \citet{thune2019nonstochastic} and has been leveraged by \citet{gyorgy21} to show that \dadaexp can achieve a refined high-probability regret bound---see \citet[Theorem~5.1]{gyorgy21}.
As a consequence, we can indeed provide an improved bound in our setting by following similar steps as in the proof of \Cref{theorem:metaif}.
The only main change is the adoption of the version of \dadaexp that uses the skipping procedure.
\begin{corollary}
    Let $\delta \in (0,1)$. If we run \metaif with \dadaexp with skipping \citep[Theorem~5.1]{gyorgy21}, then the regret of \metaif in the \setting setting with adversarial action-state mappings and stochastic losses with probability at least $1-\delta$ satisfies  
    \begin{align*}
        R_T = \bigO\biggl(& \sqrt{KT C_{K,\delta}} + \sqrt{ST \ln\frac{ST}{\delta}} + \ln\frac{1}{\delta} \\
        &+ \sqrt{C_{K,\delta} \ln K} \min_{R \subseteq \Phi}\left\{\abs{R} + \sqrt{\calD_{\Phi \setminus R} \ln K}\right\} \biggr)\;.
    \end{align*}
\end{corollary}
This result could also be extended in a similar way to \metaifswitch, so as to achieve the best result from the presence of intermediate feedback.

So far, we have provided some high-probability guarantees for the regret of both \metaif and \metaifswitch, by which we can derive some expectation bounds as well (e.g., by setting $\delta \approx 1/T$).
However, using the empirical mean estimators $\hat\theta_t$ as the mean loss estimators at time $t$ and working directly with the expected regret allows us to improve the achievable bound by a polylogarithmic factor.
Hence, for the expected regret we use \tsallisinf \citep{zimmert20}, a learning algorithm for the standard delayed bandit problem that uses a hybrid regularizer to deal with delays and gives a minimax-optimal expected regret bound.
The proof of this expected regret upper bound is in Appendix~\ref{app:expectedregret}.
\begin{restatable}{reprop}{expectedupperbound}\label{theorem:expectedregret}
    If we execute \metaifswitch with \tsallisinf \citep{zimmert20}, and use the switching condition
    $
        \sqrt{8\frakD_t \ln K} > 6\sqrt{ST\ln(2ST)}
    $
    at each round $t \in [T]$, where $\frakD_t = \sum_{j=1}^t \sigma_j$, then the regret of \metaifswitch in the \setting setting with adversarial action-state mappings and stochastic losses satisfies
    \begin{align*}
        \bbE\lrs{R_T} \leq\; &4\sqrt{2KT} + \sqrt{8\calD_{\Phi} \ln K}\\
        &+ 2\min\lrc*{6\sqrt{ST \ln(2ST)}, \sqrt{8\calD_T \ln K}} \enspace.
    \end{align*}
\end{restatable}

\begin{remark}\label{remark:a-priori-knowledge}
In \metaif, we can replace $T$ by $t^2$ in the definition of the confidence intervals for~\eqref{eq:unbiased-estimates} and remove the need for prior knowledge of the time horizon $T$.
In \metaifswitch, we could use a doubling trick to avoid the prior knowledge of $T$ in the switching condition.
On the other hand, it is not required to know the number of states $S$ for expectation bounds on the regret of \metaif.
However, removing the prior knowledge of $S$ in the high-probability regret bounds is challenging.
Indeed, to the best of our knowledge, there is no result in BIO that avoids prior knowledge on the number of states.
Lifting this requirement in the high-probability analysis is thus an interesting question for future work.
\end{remark}

\section{Lower Bounds}
\label{sec:lower}
The lower bounds in this section are for the expected regret $\E{R_T}$.
Since our algorithms provide high-probability guarantees, the upper bounds also apply to the expected regret. Throughout this section we will make use of constant delay i.e. $d_t = d$ for all $t \in [T]$.
We will first prove a general $\sqrt{KT}$ lower bound for all algorithms in \setting, after which we specialize to particular cases.

We start by proving a $\Omega\big(\sqrt{KT}\big)$ lower bound for any algorithm in our setting and for any combination of stochastic or adversarial action-state mappings and loss vectors. The construction is a reduction to the standard bandits lower bound construction.
\begin{restatable}{rethm}{thSQRTKT}\label{th:KTlowerB}
    Irrespective to whether the action-state mappings and loss vectors are stochastic or adversarial, there exists a sequence of losses such that any (possibly randomized) algorithm in \setting suffers regret
    $
           \E{R_T} = \Omega\big(\sqrt{KT}\big)
    $.
\end{restatable}
\begin{proof}
     Our construction only uses two states $h_1$ and $h_2$. The loss vectors, which are deterministic and do not change over time, are defined as follows: $\ell_t(h_1) = 1$ and $\ell_t(h_2) = 0$ for all $t\ge 0$. The stochastic action-state mapping, which is also constant over time, is given by 
    \begin{equation*}
        s_t(a) = \begin{cases}
            h_1 & ~\text{with probability $p_a$} \\
            h_2 & ~\text{with probability $1 - p_a$}
        \end{cases}
    \end{equation*}
    for all $a\in\mathcal{A}$ and $t \ge 0$, where the probabilities $p_a$ are to be determined. Thus, the loss of an arm $a$ is $\ell_t(s_t(a)) = \ell_t(h_1) = 1$ with probability $p_a$ and $\ell_t(s_t(a)) = \ell_t(h_2) = 0$ with probability $1 - p_a$. Since the loss is determined by the state, the learner receives bandit feedback without delay. We can then choose $p_a$ for $a\in\mathcal{A}$ to mimic the standard $\Omega\big(\sqrt{KT}\big)$ distribution-free bandit lower bound---e.g., see \citet[Chapter~2]{slivkins2019introduction}. By Yao's minimax principle, the same lower bound also applies to the case with adversarial action-state mappings. Since the loss vectors are deterministic, this covers all possible cases in \setting.
\end{proof}

\textbf{Adversarial action-state mapping and stochastic losses.}
We first prove a lower bound $\sqrt{ST}$ for any number $K \ge 2$ of actions. However, we do need a minor generalization of our setting to allow correlation between unseen losses. Specifically, we allow all pairs of losses $\ell_j(s), \ell_{j'}(s')$ of distinct states $s \neq s'$ to be correlated if $j > j'$ and $j - j' \leq d$, while we guarantee the i.i.d.\ nature of losses for any fixed state. Since $\E{\ell_t(S_t)} = \E{\theta(S_t)}$, this does not affect the analysis for the upper bound on the regret of our algorithms since $\E{R_T} \le \E{\calR_T}$ (see \Cref{lem:regret-notions-comparison-expectation}).
However, for a high-probability upper bound, we need to relate $R_T$ and $\mathcal{R}_T$, which now leads to an additive $\tilde{\bigO}\lr{\sqrt{ST}}$ term rather than an additive $\tilde{\bigO}\lr{\sqrt{T}}$ term as in \Cref{eq:regretrelations}.

In the proof of the $\sqrt{ST}$ lower bound, we leverage the fact that losses are independent only across time steps for a fixed state, while they may depend on the losses of the other states.
Note that our lower bound holds even when the learner knows the action-state assignments beforehand.
\begin{restatable}{rethm}{rethmstateslbwithdependence}\label{th:STlowerB}
Suppose that the action-state mapping is adversarial and the losses are stochastic and that $d_t = d$ for all $t \in [T]$. If $T \ge \min\{S, d\}$ then there exists a distribution of losses and a sequence of action-state mappings such that any (possibly randomized) algorithm suffers regret
$
    \E{R_T} = \Omega\big(\!\sqrt{\min\{S, d\}T}\big)
$.
\end{restatable}
We provide a sketch of the proof of Theorem~\ref{th:STlowerB} (see Appendix~\ref{app:lowerbounds} for the full proof). First, suppose that $S \leq 2d$. For the construction of the lower bound we only consider two actions and equally split the states over these two actions. Then, we divide the $T$ time steps in blocks of length $S/2 \le d$. In each block, each state has the same loss. Since the block length is smaller then the delay, we have effectively created a two-armed bandit problem with $T' = T/(S/2)$ rounds and loss range $[0,S/2]$, for which we can prove a $\Omega\big(S\sqrt{T'}\big) = \Omega\big(\sqrt{ST}\big)$ lower bound by showing an equivalent lower bound for the full information setting. If $S > 2d$, we use the same construction with only $2d$ states, and obtain a $\Omega\big(\sqrt{d T}\big)$ lower bound. 

Finally, we can show the following lower bound, whose proof can be found in \Cref{app:lowerbounds}.
\begin{restatable}{rethm}{rethmlbfixeddelay} \label{thm:lb-fixed-delay}
    Suppose that the action-state mapping is adversarial, the losses are stochastic, and that $d_t = d$ for all $t \in [T]$. If $T \ge d + 1$ then there exists a distribution of losses and a sequence of action-state mappings such that any (possibly randomized) algorithm suffers regret
    \[
        \E{R_T} = \Omega\lr*{\min\lrc*{(d+1)\sqrt{S}, \sqrt{(d+1)T}}} \enspace.
    \]
\end{restatable}
This term is also present in the dynamic regret bound of \nsd, but it is necessarily incurred from their analysis even in the stationary case \citep[Theorem~1]{Vernade0M20}.

This last lower bound implies that the regret of our algorithm is near-optimal. Since the lower bound of Theorem~\ref{th:KTlowerB} applies to the case where the action-state mapping is adversarial and the losses are stochastic, we find the following result as a corollary of Theorem~\ref{th:KTlowerB},  Theorem~\ref{th:STlowerB}, and Theorem~\ref{thm:lb-fixed-delay}. 
\begin{corollary}
    Suppose that the action-state mapping is adversarial, the losses are stochastic, and that $d_t = d$ for all $t \in [T]$. If $T \ge 1 + \min\{S, d\}$, then there exists a distribution of losses and a sequence of action-state mappings such that any (possibly randomized) algorithm suffers regret 
\[
    \E{R_T} = \Omega\big(\!\max\big\{\sqrt{KT}, \sqrt{\min\{S, d\}T}, (d+1)\sqrt{S}\big\}\big) \;.
\]
\end{corollary}
\textbf{Stochastic action-state mappings and adversarial losses.}
In this case we recover the standard lower bound for adversarial bandits with bounded delay.
\begin{restatable}{rethm}{thm:stoadvlowerbound}\label{th:advstochlowerbound}
    Suppose that the action-state mapping is stochastic, the losses are adversarial, and that $d_t = d$ for all $t \in [T]$. Then there exists a stochastic action-state mapping and a sequence of losses such that any (possibly randomized) algorithm suffers regret 
    $
        \E{R_T} = \Omega\big(\!\max\big\{\sqrt{KT}, \sqrt{d T}\big\}\big)
    $.
\end{restatable}
\begin{proof}
    Since by Theorem~\ref{th:KTlowerB} we already know that any algorithm must suffer $\Omega\big(\sqrt{KT}\big)$ regret, we only need to show a $\Omega(\sqrt{d T})$ lower bound. We use two states, $h_1$ and $h_2$. Our  action-state mapping is deterministic and, for all $t \ge 0$, assigns $s_t(a) = h_1$ to all but one action $a^\star$, to which the mapping assigns $s_t(a^\star) = h_2$. We now have constructed a two-armed bandit problem with delayed feedback and $T$ rounds, for which a $\Omega(\sqrt{d T})$ lower bound is known \citep{cesa2019delay}.
\end{proof}

\textbf{Adversarial action-state mappings, adversarial losses.} Since we can recover the construction of the lower bound in Theorem~\ref{th:advstochlowerbound}, we have the following result. 
\begin{corollary}
    Suppose that the action-state mapping is adversarial, the losses are adversarial, and that $d_t = d$ for all $t \in [T]$. Then there exists an action-state mapping and a sequence of losses such that any (possibly randomized) algorithm suffers regret 
    $
        \E{R_T} = \Omega\big(\!\max\big\{\sqrt{KT}, \sqrt{d T}\big\}\big)
    $.
\end{corollary}

\section{Experiments}
\label{sec:experiments}

We empirically compare our algorithm \metaif with the following baselines: \dadaexp \citep{gyorgy21} for adversarial delayed bandits without intermediate observations (which we used to instantiate the algorithm $\calB$), the standard \ucb algorithm \citep{AuerCF02} for stochastic bandits without delays and intermediate observations, and \nsd \citep{Vernade0M20} for nonstationary stochastic action-state mappings and stochastic losses.
We run all experiments with a time horizon of $T=10^4$.
All our plots show the cumulative regret of the algorithms considered as a function of time.
The performance of each algorithm is averaged over $20$ independent runs in every experiment, and the shaded areas consider a range centered around the mean with half-width corresponding to the empirical standard deviation of these $20$ repetitions.
In the first two experiments, we consider both fixed delays $d \in \{50, 100, 200\}$ and random delays $d_t \sim \mathrm{Laplace}(50,25)$ sampled i.i.d.\ from the Laplace distribution with $\E{d_t} = 50$.

\begin{figure}[h]
    \centering
    \begin{subfigure}[t]{0.48\linewidth}
         \centering
         \caption{$d = 50$}
         \label{fig:sto_delay_50}
         \includegraphics[width=\textwidth]{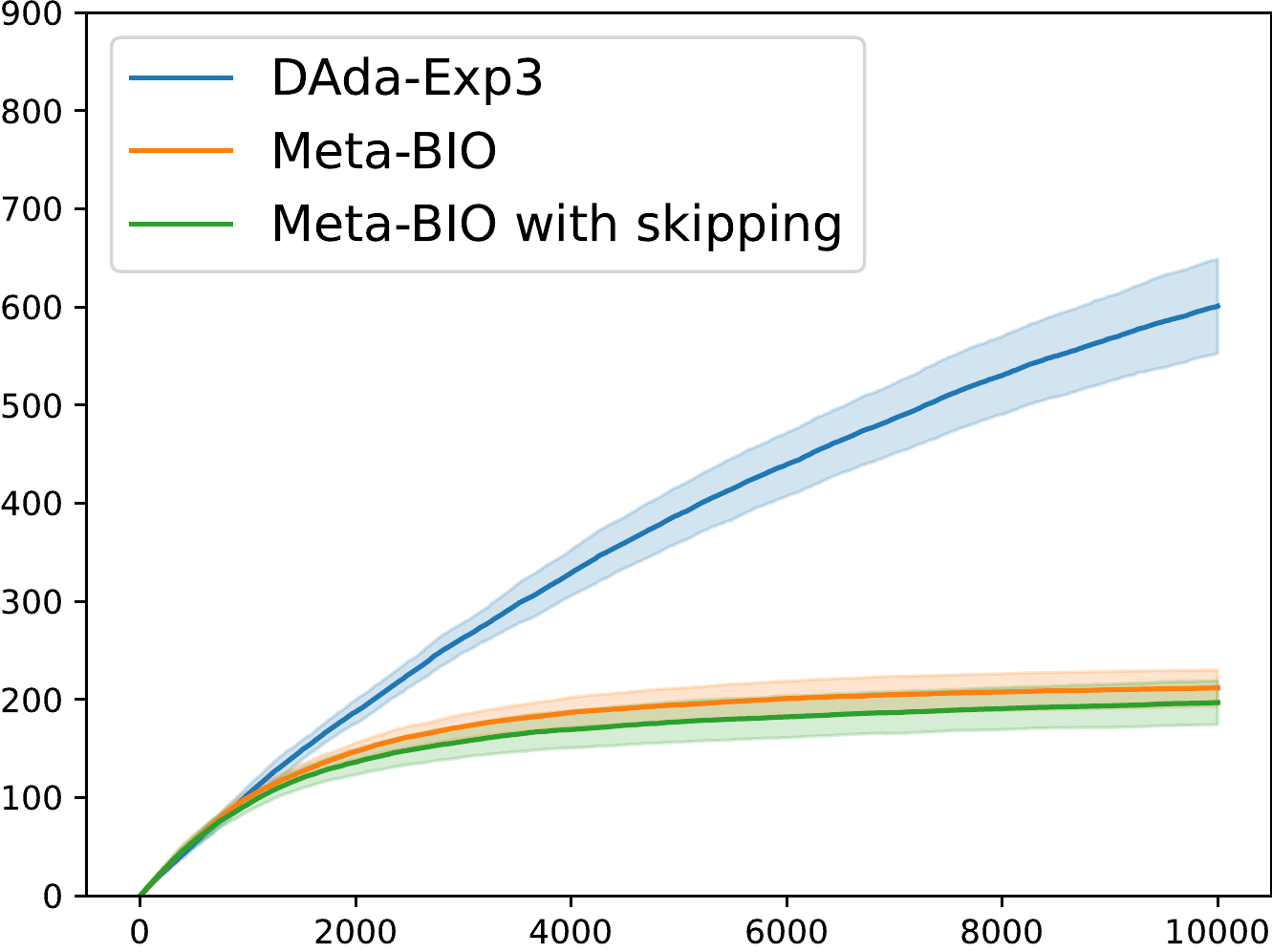}
    \end{subfigure}
    \begin{subfigure}[t]{0.48\linewidth}
         \centering
         \caption{$d = 100$}
         \label{fig:sto_delay_100}
         \includegraphics[width=\textwidth]{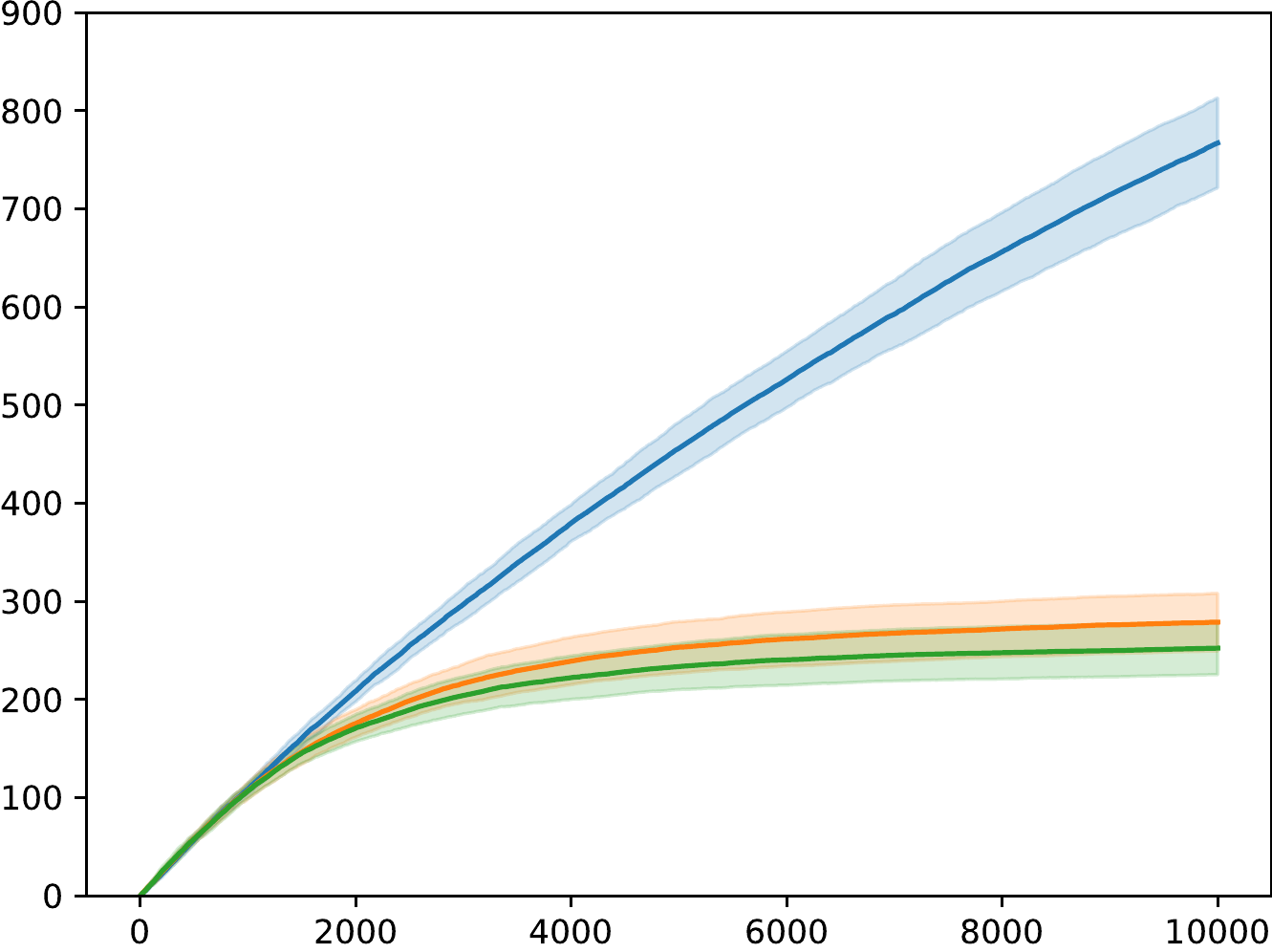}
    \end{subfigure}
    \hfill
    \begin{subfigure}[t]{0.48\linewidth}
         \centering
         \caption{$d = 200$}
         \label{fig:sto_delay_200}
         \includegraphics[width=\textwidth]{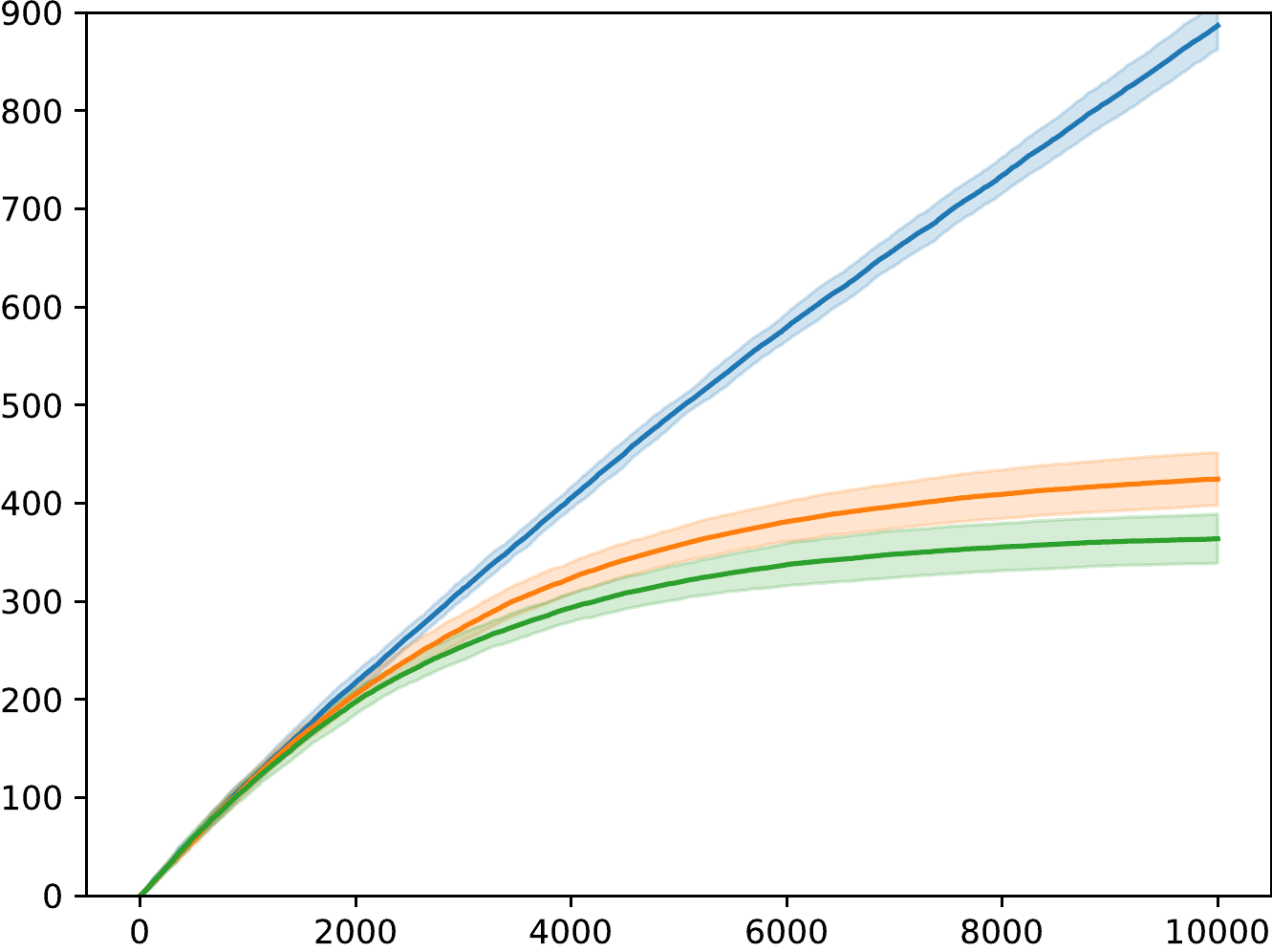}
    \end{subfigure}
    \begin{subfigure}[t]{0.48\linewidth}
         \centering
         \caption{$d_t \sim \mathrm{Laplace}(50,25)$}
         \label{fig:sto_delay_random}
         \includegraphics[width=\textwidth]{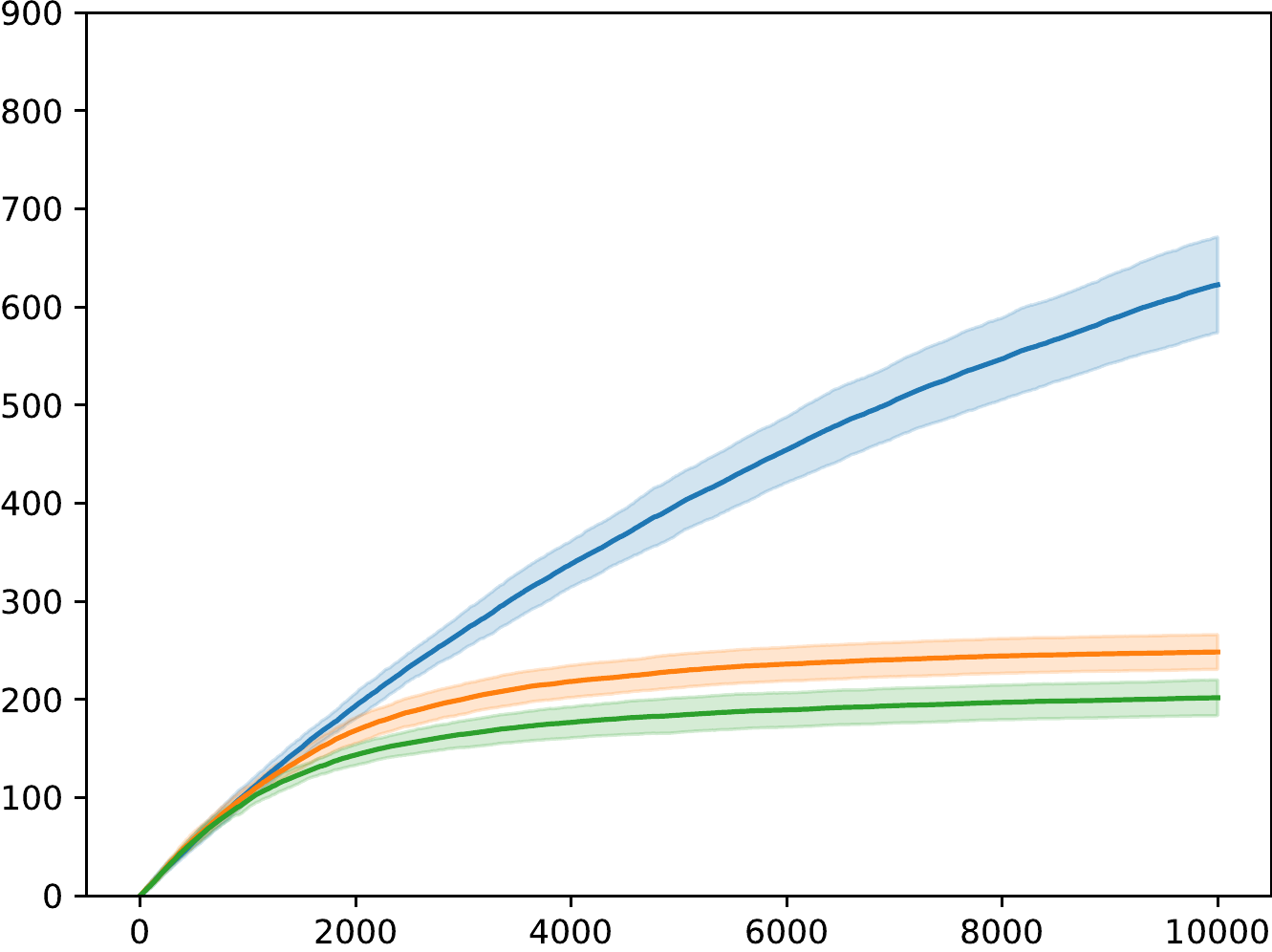}
    \end{subfigure}
    \caption{Cumulative regret over time for the stochastic action-state mapping when delays are fixed or random.}
    \label{fig:stoc_transition}
\end{figure}

\textbf{Experiment 1: stochastic action-state mappings.}
Here we use a stationary version of the experiments in \citep{Vernade0M20}---see \Cref{table:stochastic_action_state_map} in \Cref{app:mappings} for details.
We set $K=4$ and $S=3$, while we repeat this experiment for the previously mentioned values of delays.
\Cref{fig:stoc_transition} shows that, across all delay regimes, \metaif largely improves on the performance of \dadaexp by exploiting intermediate observations.

\textbf{Experiment 2: adversarial action-state mappings.}
In this construction, we simulate the adversarial mapping using a construction adapted from \citep{ZimmertS21}: we alternate between two stochastic mappings while keeping the loss means fixed.
We set $K=4$, $S=3$, and we consider multiple instances for the different values of delays as in the previous experiment.
The interval between two consecutive changes in the distribution of action-state mappings grows exponentially.
See \Cref{table:adversarial_action_state_map} in \Cref{app:mappings} for details.
\Cref{figure:adversarial_map} shows that \metaif and \metaif with ``skipping'' outperform both \ucb and \nsd.

\begin{figure}[h]
    \centering
    \begin{subfigure}[t]{0.48\linewidth}
         \centering
         \caption{$d = 50$}
         \label{fig:adv_delay_50}
         \includegraphics[width=\textwidth]{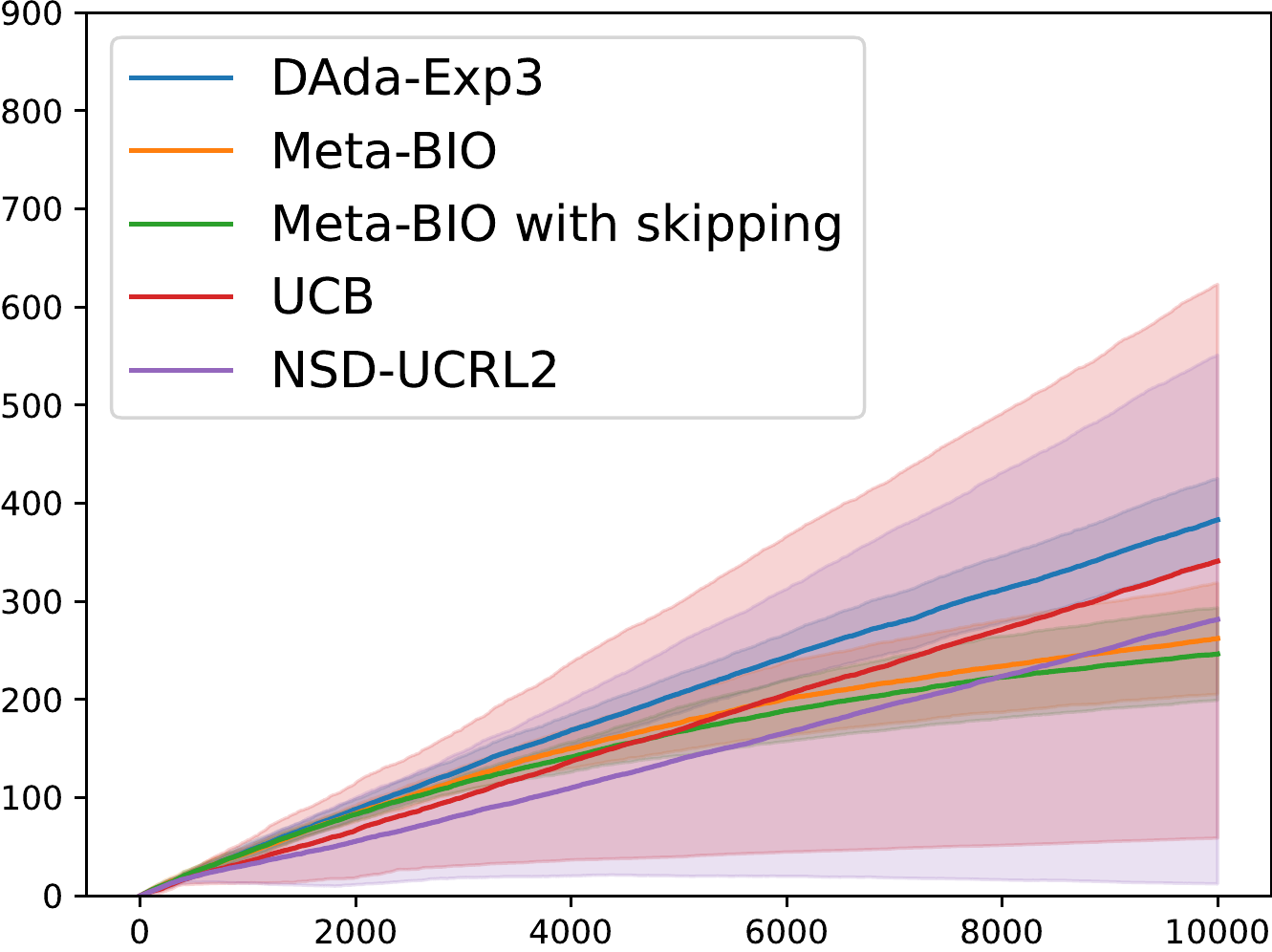}
     \end{subfigure}
    \begin{subfigure}[t]{0.48\linewidth}
         \centering
         \caption{$d = 100$}
         \label{fig:adv_delay_100}
         \includegraphics[width=\textwidth]{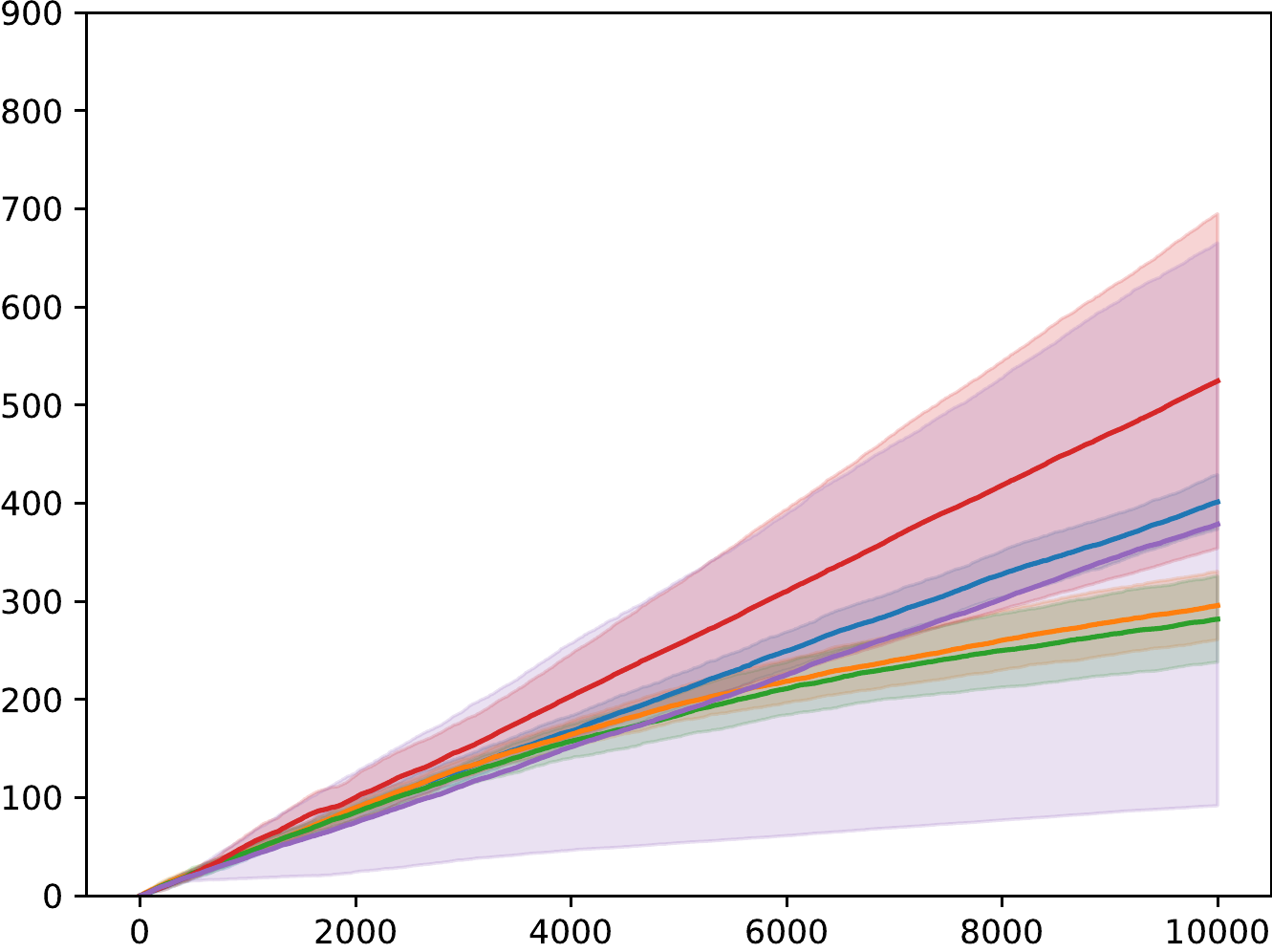}
     \end{subfigure}
    \hfill
    \begin{subfigure}[t]{0.48\linewidth}
         \centering
         \caption{$d = 200$}
         \label{fig:adv_delay_200}
         \includegraphics[width=\textwidth]{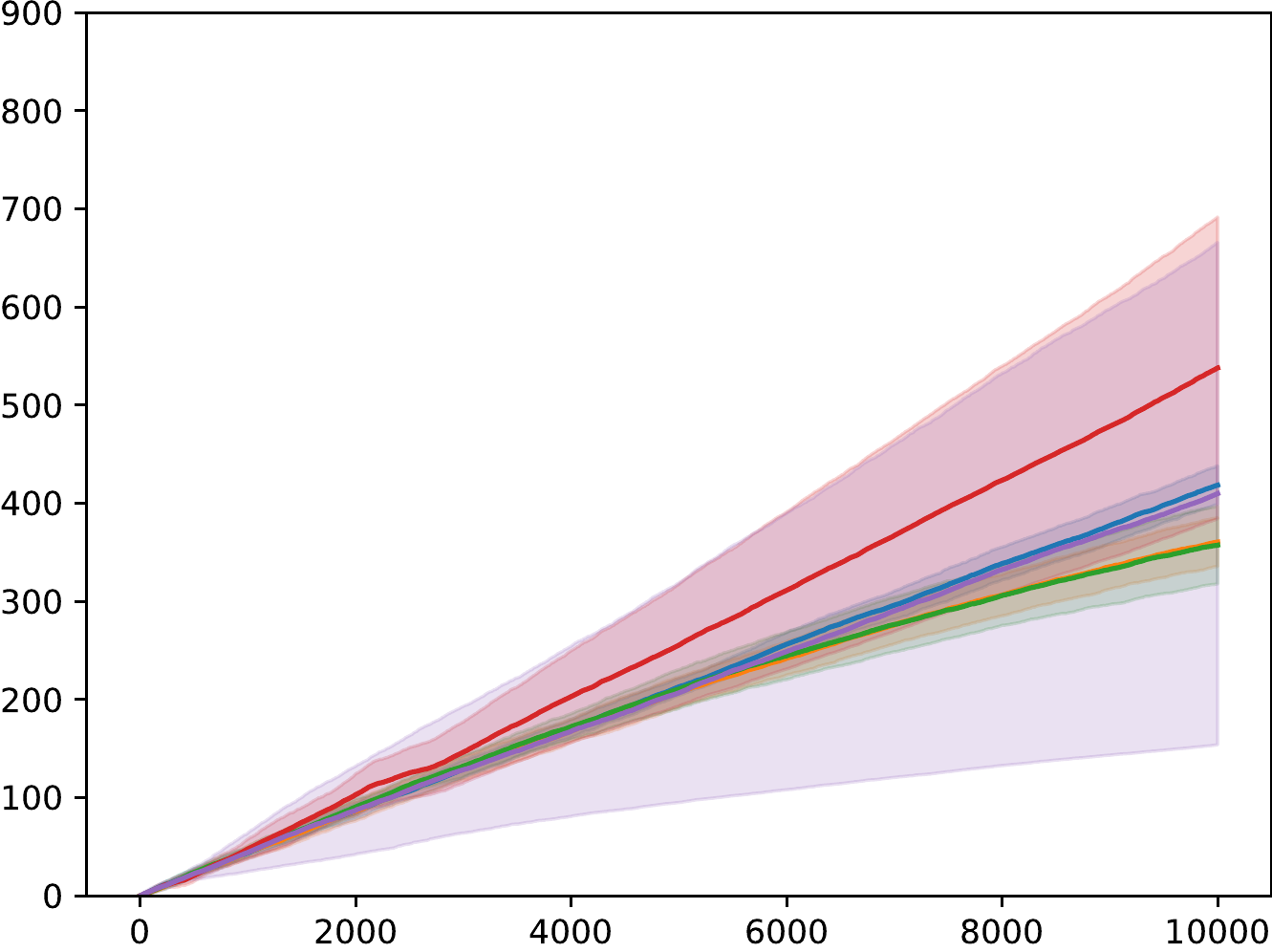}
    \end{subfigure}
    \begin{subfigure}[t]{0.48\linewidth}
         \centering
         \caption{$d_t \sim \mathrm{Laplace(50,25)}$}
         \label{fig:adv_delay_random}
         \includegraphics[width=\textwidth]{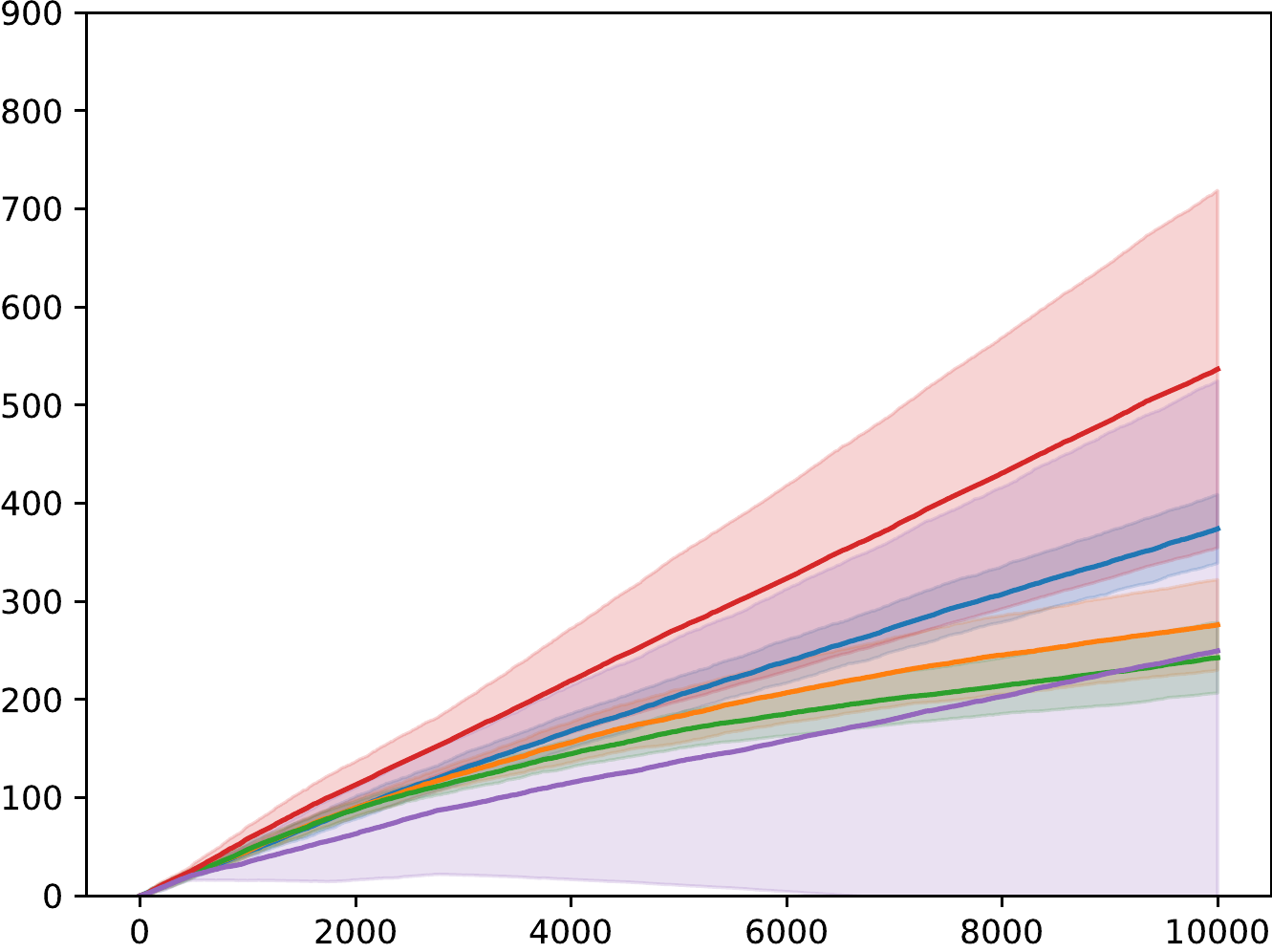}
     \end{subfigure}
    \caption{Cumulative regret over time for the adversarial action-state mapping when delays are fixed or random. All algorithms have small variance except for \ucb and \nsd.}
    \label{figure:adversarial_map}
\end{figure}

\textbf{Experiment 3: utility of intermediate observations.}
Here we set $K = 8$, $d = 100$, and investigate how the performance of \metaif changes when the number $S$ of states varies in $\{4,6,8,10,12\}$.
The mean loss is always $0.2$ for the optimal state and $1$ for the others.
The optimal action always maps to the optimal state.
The suboptimal actions map to the optimal state with probability $0.6$ and map to a random suboptimal state with probability $0.4$.
This implies that the expected loss of each arm remains constant when the number of states changes.
\Cref{fig:state} shows that the regret gap between \metaif and \dadaexp shrinks as the number of states increases.
This observation confirms our theoretical findings about the dependency of the regret on the number of states, which lead to a larger improvement the fewer they are.

\begin{figure}[ht]
    \centering
    \includegraphics[width=0.38\textwidth]{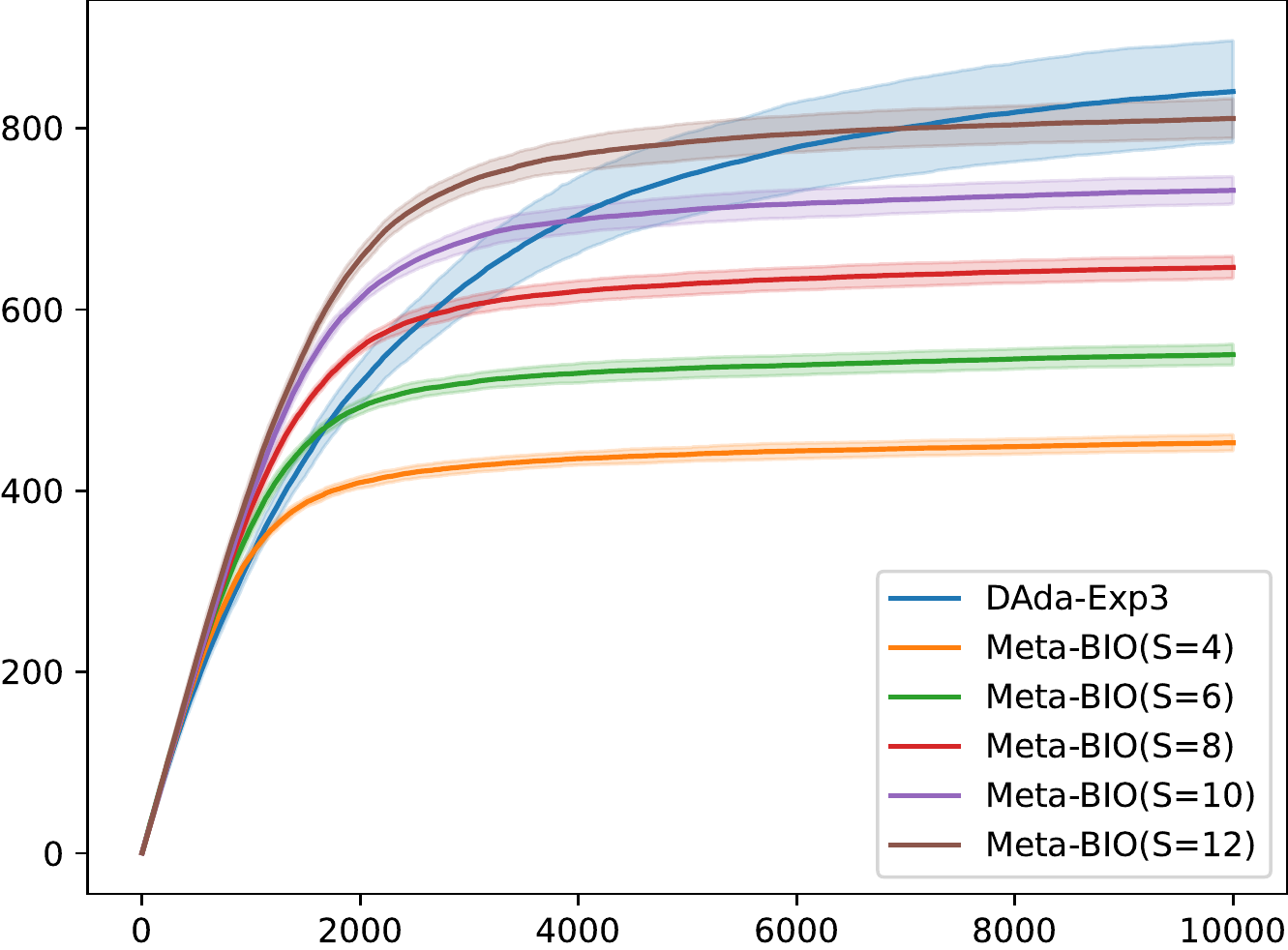}
    \caption{Cumulative regret over time of both \dadaexp and \metaif with different numbers of states $S \in \{4,6,8,10,12\}$.}
    \label{fig:state}
\end{figure}

\textbf{Experiment 4: performance of \metaifswitch when $S<d$.}
We use the same setting as in Experiment~1 with delay $d=20$.\footnote{Compared to the switching condition used for the analysis of \metaifswitch, we replace $49ST\ln\frac{8ST}{\delta}$ with $ST$. This change allows the switching condition to be triggered more easily to provide a better visualization of the behaviour of \metaifswitch, while it only introduces a polylog factor in its regret bound.}
\Cref{fig:switch} shows the performance of \metaifswitch compared with both \dadaexp and \metaif.
Before the switching point, \metaifswitch runs \dadaexp (up to independent internal randomization). Afterwards, \metaifswitch switches to \metaif (which in turn runs \dadaexp as a subroutine) and quickly aligns with its performance.
Note that, at the switching time, \metaifswitch uses (via \metaif) the same instance of \dadaexp that was already running, rather than starting a new instance. It can be shown that our analysis of \metaifswitch applies to this variant as well without changes in the order of the bound.

\begin{figure}[h]
    \centering
    \includegraphics[width=0.38\textwidth]{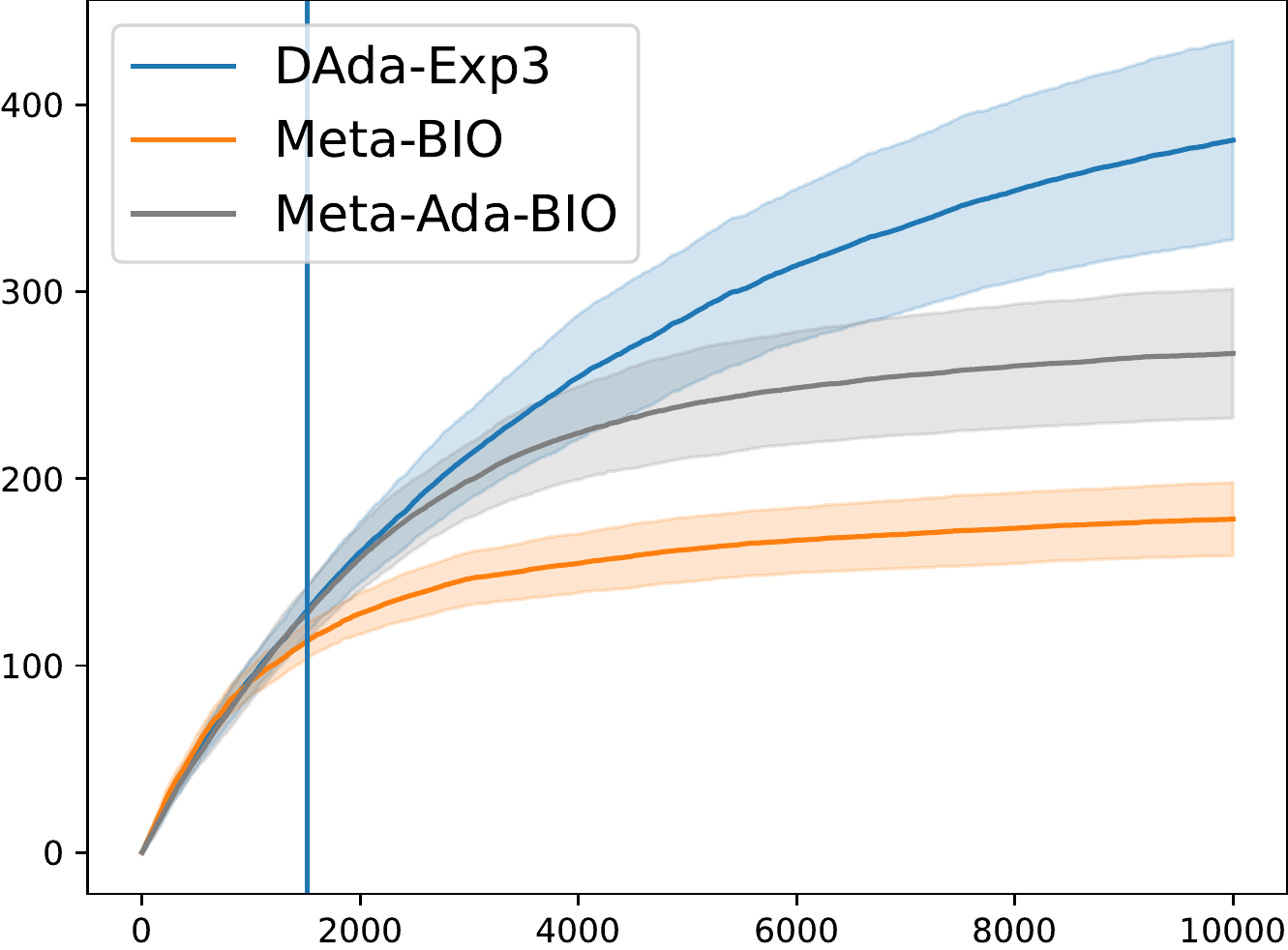}
    \caption{Cumulative regret over time of \dadaexp, \metaif and \metaifswitch. The vertical blue line marks the switching point of \metaifswitch.}
    \label{fig:switch}
\end{figure}

\paragraph{Experiment 5: performance of \metaifswitch when $S>d$.}
We use a setting that is almost identical to that of Experiment~3 (\Cref{sec:experiments}), except we set $d = 4$ and $S = 14$. The performance of the three algorithms is shown in \Cref{fig:switch_add}. We can observe that \metaifswitch does not switch to \metaif and its performance is thus the same as that of \dadaexp, whereas \metaif incurs a larger regret.

\begin{figure}[h]
    \centering
    \includegraphics[width=0.38\textwidth]{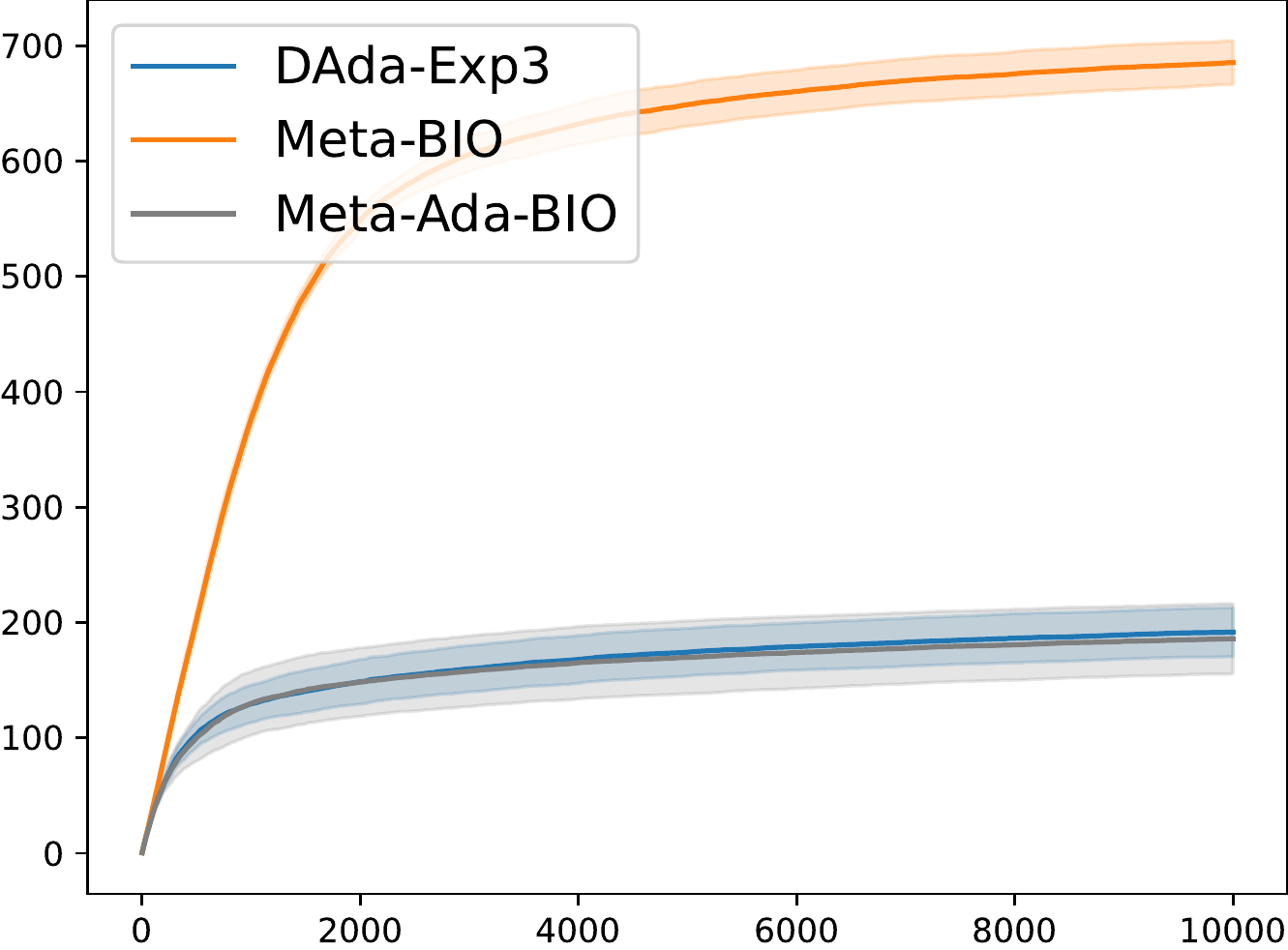}
    \caption{Cumulative regret over time of \dadaexp, \metaif and \metaifswitch when $S>d$.}
    \label{fig:switch_add}
\end{figure}

\section{Future Work}
\label{sec:discussion}

The work of \citet{Vernade0M20} also considers a non-stationary action-state mapping and derive regret bounds for the switching regret. Preliminary results suggest that, as long as there is an algorithm that can provide bounds on the switching regret with delayed feedback, our ideas also transfer to this setting. Unfortunately, there is currently no algorithm that can provide bounds on the switching regret with delayed feedback and we leave this as a promising direction for future work. 

\section*{Acknowledgements}
EE, NCB, and HQ are partially supported by the MIUR PRIN grant Algorithms, Games, and Digital Markets (ALGADIMAR), by the EU Horizon 2020 ICT-48 research and innovation action under grant agreement 951847, project ELISE (European Learning and Intelligent Systems Excellence), and by the FAIR (Future Artificial Intelligence Research) project, funded by the NextGenerationEU program within the PNRR-PE-AI scheme (M4C2, investment 1.3, line on Artificial Intelligence). SM acknowledges funding from the European Union's Horizon 2020 research and innovation programme under the Marie Skłodowska-Curie grant agreement No 801199. YS acknowledges partial support by the Independent Research Fund Denmark, grant number 9040-00361B. This work was mostly done while DvdH was at the University of Milan partially supported by the MIUR PRIN grant Algorithms, Games, and Digital Markets (ALGADIMAR) and partially supported by Netherlands Organization for Scientific Research (NWO), grant number VI.Vidi.192.095.

\bibliography{refs}
\bibliographystyle{icml2023}

\newpage
\appendix
\onecolumn

\section{Auxiliary Results}

\begin{lemma} \label{lem:regret-notions-comparison}
    Consider any algorithm that picks actions $\lr{A_t}_{t\in[T]}$ in the adversarial delayed bandits problem with intermediate feedback with arbitrary action-state mappings $\lr{s_t}_{t\in[T]}$ and i.i.d.\ loss vectors $\lr{\ell_t}_{t\in[T]}$.
    Then, for any given $\delta \in (0,1)$,
    \[
        R_T - \calR_T \le \sqrt{2T\ln(2/\delta)}
        \qquad \text{and} \qquad
        \calR_T - R_T \le \sqrt{2T\ln(2K/\delta)}
    \]
    individually hold with probability at least $1-\delta$.
\end{lemma}
\begin{proof}
    First, observe that we can relate the two notions of regret as
    \[
        R_T = \calR_T +
        \sum_{t=1}^T \bigl(\theta(S_t) - \ell_t(S_t)\bigr) +
        \underbrace{\min_{a \in \calA} \sum_{t=1}^T \ell_t(s_t(a)) - \min_{a \in \calA} \sum_{t=1}^T \theta(s_t(a))}_{(\triangle)} \enspace.
    \]
    By Azuma-Hoeffding inequality, we can show that each side of 
    \begin{equation}
        -\sqrt{\frac{T}{2}\ln\lr*{\frac{1}{\delta'}}} \le \sum_{t=1}^T \bigl(\theta(S_t) - \ell_t(S_t)\bigr) \le \sqrt{\frac{T}{2}\ln\lr*{\frac{1}{\delta'}}}
    \end{equation}
    holds with probability at least $1-\delta'$.
    Now, define
    \[
        a^*_{\ell} \in \argmin_{a \in \calA} \sum_{t=1}^T \ell_t(s_t(a)) \qquad \text{and} \qquad a^*_{\theta} \in \argmin_{a \in \calA} \sum_{t=1}^T \theta(s_t(a)) \enspace.
    \]
    On the one hand, observe that
    \[
        (\triangle) \le \sum_{t=1}^T \ell_t(s_t(a^*_{\theta})) - \sum_{t=1}^T \theta(s_t(a^*_{\theta})) \le \sqrt{\frac{T}{2}\ln\lr*{\frac{1}{\delta'}}} \enspace,
    \]
    where the last inequality holds with probability at least $1-\delta'$ by Azuma-Hoeffding inequality.
    On the other hand, we can show that
    \[
        (\triangle) \ge \sum_{t=1}^T \ell_t(s_t(a^*_{\ell})) - \sum_{t=1}^T \theta(s_t(a^*_{\ell})) =: (\diamond) \enspace.
    \]
    However, in this case $a^*_{\ell}$ depends on the entire sequence $\ell_1, \ldots, \ell_T$.
    We thus need to use a union bound in order to show that
    \[
        \pr{(\diamond) \le -\sqrt{\frac{T}{2}\ln\lr*{\frac{K}{\delta'}}}}
        \le \sum_{a \in \calA} \pr{\sum_{t=1}^T \ell_t(s_t(a)) - \sum_{t=1}^T \theta(s_t(a)) \le -\sqrt{\frac{T}{2}\ln\lr*{\frac{K}{\delta'}}}}
        \le \delta' \enspace,
    \]
    where the last inequality follows by Azuma-Hoeffding inequality.
    We conclude the proof by setting $\delta'=\delta/2$.
\end{proof}

\begin{lemma} \label{lem:concentration-unbiased-estimates}
    The estimates $(\hat{\theta}_t)_{t=1}^T$ defined in \Cref{eq:unbiased-estimates} are such that $\abs{\hat{\theta}_t(s) - \theta(s)} \le \frac12 \e_t(s)$ simultaneously holds for all $t \in [T]$ and all $s \in \calS$ with probability at least $1-\delta/2$.
\end{lemma}
\begin{proof}
    In a similar way as in \citet{Vernade0M20}, define $X_m(s)$ to be the empirical mean estimate for $\theta(s)$ which uses the first $m \in [T]$ observed losses corresponding to state $s \in \calS$.
    Notice that $\hat\theta_t(s) = X_{N'_t(s)}(s)$, while we define $\e'_m(s) = \sqrt{\frac{2}{m} \ln\lr{\frac{4ST}{\delta}}}$ so that $\e_t(s) = \e'_{N'_t(s)}(s)$.
    We can additionally observe that $\E{X_m(s)} = \theta(s)$.
    Then, we can use Azuma-Hoeffding inequality to show that
    \begin{align*}
        \pr{\bigcap_{s \in \mathcal{S}}\bigcap_{t \in [T]} \lrc*{\abs{\hat\theta_t(s) - \theta(s)} \le \frac12 \e_t(s)}}
        &\ge \pr{\bigcap_{s \in \mathcal{S}}\bigcap_{m \in [T]} \lrc*{\abs{X_m(s) - \theta(s)} \le \frac12 \e'_m(s)}} \\
        &\ge 1 - 2 \sum_{s \in \mathcal{S}}\sum_{m=1}^T e^{-\frac12 \e'_m(s)^2 m} \\
        &= 1-\frac{\delta}{2} \enspace,
    \end{align*}
    where we also used a union bound in the second inequality.
\end{proof}

\begin{lemma} \label{lem:regret-notions-comparison-expectation}
    Consider any algorithm that picks actions $\lr{A_t}_{t\in[T]}$ in the \setting setting with adversarial action-state mappings $\lr{s_t}_{t\in[T]}$ and stochastic loss vectors $\lr{\ell_t}_{t\in[T]}$.
    Assume that the losses for any fixed state are i.i.d., whereas pairs of losses $\ell_j(s), \ell_{j'}(s')$ of distinct states $s \neq s'$ might be correlated when $j>j'$ and $j-j'\le d_{j'}$.
    Then, it holds that $\E{R_T} \le \E{\calR_T}$, where the expectation is with respect to the stochasticity of the losses and the randomness of the algorithm.
\end{lemma}
\begin{proof}
    We know that $\E{\ell_t(s_t(a))} = \theta(s_t(a))$ for any fixed $a \in \calA$ and all $t \in [T]$.
    We further observe that
    \[
        \E{\ell_t(S_t)} = \bbE\Bigl[\E{\ell_t(s_t(A_t)) \;\middle|\; A_t}\Bigr] = \E{\theta(S_t)}
    \]
    holds for all $t \in [T]$, as $A_t$ is independent of losses that can be correlated with $\ell_t$.
    Now, define
    \[
        a^*_\ell \in \argmin_{a\in \calA} \sum_{t=1}^T \ell_t(s_t(a))
        \qquad \text{and} \qquad
        a^*_\theta \in \argmin_{a\in \calA} \sum_{t=1}^T \theta(s_t(a)) \enspace.
    \]
    Then, we conclude the proof by showing that
    \begin{align*}
        \E{\calR_T} &= \sum_{t=1}^T \E{\ell_t(S_t)} - \E{\sum_{t=1}^T \ell_t(s_t(a^*_\ell))} \\
        &\ge \sum_{t=1}^T \E{\ell_t(S_t)} - \E{\sum_{t=1}^T \ell_t(s_t(a^*_\theta))}
        = \sum_{t=1}^T \E{\theta(S_t)} - \sum_{t=1}^T \theta(s_t(a^*_\theta)) = \E{R_T} \enspace.
    \end{align*}
\end{proof}

\section{High-Probability Regret Bound}

\subsection{Total delay bound}\label{app:delaybound}

\relemmatotaldelaybound*
\begin{proof}[Proof of \Cref{lem:total-delay-bound}]
    For any $s \in \calS$, we define $\calT_s = \lrc{t \in [T]: S_t = s}$ to be the set of all rounds when the state observed by the learner corresponds to $s$.
    Denote by $t_s$ the last time step $t \in \calT_s$ such that $N_{t}(s) < \sigma_{t}$ and let $\calC_s = \lrc{t \in \calT_s : t \le t_s}$ be those rounds in $\calT_s$ that come no later than $t_s$.
    According to the choice of $t_s$, all the rounds in $\calT_s$ for which learner waits for the respective delayed loss, must belong to $\calC_s$, while the learner incurs $\tilde d_t = 0$ delay for rounds $t \in \calT_s \setminus \calC_s$.
    Now we partition $\calC_s$ into two sets: the observed set $\calC_s^{\mathrm{obs}} = \lrc{t \in \calC_s: t + d_t \le t_s}$ and the outstanding set $\calC_s^{\mathrm{out}} = \lrc{t \in \calC_s: t + d_t > t_s}$.
    From the choice of $t_s$, we can see that the number of rounds in $\calC_s^{\mathrm{obs}}$ is
    \[
        \abs{\calC_s^{\mathrm{obs}}} \le N_{t_s}(s) < \sigma_{t_s} \le \sigma_{\max} \enspace,
    \]
    and the number of rounds in $C_s^{\mathrm{out}}$ is
    \[
        \abs{\calC_s^{\mathrm{out}}} \le \sigma_{t_s} \le \sigma_{\max} \enspace.
    \]
    Therefore, we have $|\calC_s| \leq 2 \sigma_{\max}$. So if we define $\calC_{\mathrm{all}} = \bigcup_{s \in \calS} \calC_{s}$, then $|\calC_{\mathrm{all}}| \leq \min\lrc{2S\sigma_{\max}, T} = \abs{\Phi}$.
    This also implies that
    \[
    \sum_{t = 1}^T \Tilde{d}_t \leq \sum_{t \in \calC_{\mathrm{all}}} d_t \leq \sum_{t \in \Phi} d_{t}
    \]
    by definition of $\Phi$.
\end{proof}

\subsection{Improved Regret for \dadaexp for Fixed $\delta$} \label{app:improved-dadaexp-bound}

We follow the analysis of Theorem 4.1 in \citet[Appendix~A]{gyorgy21} and our goal is to use the knowledge of $\delta \in (0,1)$ to tune the learning rates $(\eta_t)_{t \in [T]}$ and the implicit exploration terms $(\gamma_t)_{t \in [T]}$, accordingly.
Let $d_1, \ldots, d_T$ be the sequence of delays perceived by \dadaexp, and let $D_T = \sum_{t=1}^T d_t$ be its total delay.
Furthermore, let $\sigma_t$ be the number of outstanding observations of \dadaexp at the beginning of round $t \in [T]$.
Suppose that we take $\gamma_t = c \eta_t$ with $c>0$ for all $t \in [T]$, then following the same analysis as in \citet[Appendix~A]{gyorgy21}, we end up with the following regret bound that holds with probability at least $1-2\delta'$ for any $\delta' \in (0,1/2)$:
\begin{align*}
   \calR_T &\leq \frac{\ln(K)}{\eta_T} + \sum_{t = 1}^T \eta_t (\sigma_t + (c+1)K) + \frac{\ln(K/\delta')}{2c \eta_T} + \frac{\sigma_{\max} + c + 1}{2c}\ln(1/\delta')\\
   &= \frac1{\eta_T} \lr*{\ln(K) + \frac{\ln(K/\delta')}{2c}} + \sum_{t = 1}^T \eta_t (\sigma_{t-1} + (c+1)K) + \frac{\sigma_{\max} + 1}{2c}\ln(1/\delta') + \frac{\ln(1/\delta')}{2}\enspace.
\end{align*}

Therefore, by taking $\eta_t^{-1} = \sqrt{\frac{(c+1)Kt + \sum_{j = 1}^{t} \sigma_j}{2\ln(K) + \frac{1}{c}\ln(K/\delta')}}$, we get the following bound with probability at least $1-2\delta'$:
\[
    \calR_T \leq 2\sqrt{\lr*{(c+1)KT + \sum_{t = 1}^{T} \sigma_t} \lr*{2\ln(K) + \frac{\ln(K/\delta')}{c}}} + \frac{\sigma_{\max} + 1}{2c}\ln(1/\delta') + \frac{\ln(1/\delta')}{2}\enspace.
\]
We know that $\sum_{t = 1}^{T} \sigma_t = D_T$ by definition of $\sigma_t$.
Then, we can set $c = 1 $ to obtain that the regret $\calR_T$ (as per the original notion of regret used in \citet{gyorgy21}) is
\begin{equation}\label{eq:improveddada-1}
    \calR_T \leq 2\sqrt{2KT   \lr*{3\ln(K) + \ln\lr{1/\delta'}}} + 2\sqrt{D_T \lr*{3\ln(K) + \ln\lr{1/\delta'}}} + \frac{\sigma_{\max} + 2}{2}\ln(1/\delta')
\end{equation}
with probability at least $1-2\delta'$.

From \cref{lem:regret-notions-comparison}, we have that
\begin{equation}\label{eq:improveddada-2}
    R_T \leq \calR_T + \sqrt{2T \ln(2/\delta')}
\end{equation}
holds with probability at least $1-\delta'$.
So, combining \Cref{eq:improveddada-1,eq:improveddada-2}, and setting $\delta = 3\delta'$, we can upper bound our notion of regret $R_T$ as
\begin{equation} \label{eq:improveddada}
    R_T \leq 2\sqrt{2KT   \lr*{3\ln(K) + \ln\lr{3/\delta}}} +  \sqrt{2T \ln(6/\delta)} + 2\sqrt{D_T \lr*{3\ln(K) + \ln\lr{3/\delta}}}  + \frac{\sigma_{\max} + 2}{2}\ln(3/\delta)
\end{equation}
with probability at least $1-\delta$.

\subsection{Reduction to \dadaexp via \metaif} \label{app:reductiontodada}

Based on the reduction via \metaif, we require that $\calB$ guarantee a regret bound
\begin{equation} \label{eq:regret-dadaexp-estimates}
    \hat{\calR}_T^{\calB} = \sum_{t=1}^T \tilde{\theta}_t(S_t) - \min_{a \in \calA} \sum_{t=1}^T \tilde{\theta}_t(s_t(a))
\end{equation}
that holds with high probability when the losses experienced by $\calB$ are of the form $\tilde{\theta}_t\big(s_t(a)\big)$.
Note that, even though the action-state mappings $s_1, \ldots, s_T$ are unknown to the learner, we can provide those losses as long as $\calB$ requires bandit feedback only. Indeed, we can compute $\tilde{\theta}_t(S_t)$ defined in \Cref{eq:lowerconf-estimates,eq:unbiased-estimates}, while we cannot determine $s_t(a)$ for all actions $a \in \calA$ that are not $A_t$.
As mentioned in \Cref{sec:analysis}, in this work we consider \dadaexp \citep{gyorgy21} as algorithm $\calB$ used by \metaif.
In what follows, we refer to this specific choice for the algorithm $\calB$.

The analysis of \dadaexp for the high-probability bound (\Cref{theorem:dadaexp}) is such that most steps only require that the loss of each action is bounded in $[0,1]$.
Then, those steps apply for any such sequence of loss vectors.
However, the crucial part of that analysis that requires attention is the application of Lemma~1 from \citet{Neu15Implicit}.
We restate it below for reference.

Before that, we introduce the notation required for stating the result.
We consider a learner choosing actions $A_1, \ldots, A_T$ according to probability distributions $p_1, \ldots, p_T$ over actions.
We denote by $\calF_{t-1}$ the observation history of the learner until the beginning of round $t$.
The result uses importance-weighted estimates for the losses $\ell_1, \ldots, \ell_T$ with implicit exploration, where the implicit exploration parameter is $\gamma_t \ge 0$ for each time $t$.
These loss estimates are defined as
\begin{equation} \label{eq:standard-ix-loss}
    \tilde{\ell}_t(a) = \frac{\mathbbm{1}\lrs{A_t=a}}{p_t(a) + \gamma_t} \ell_t(a) \qquad \forall t \in [T], \forall a \in \calA \enspace.
\end{equation}

\begin{lemma}[{\citet[Lemma~1]{Neu15Implicit}}]
    Let $\gamma_t$ and $\alpha_{t}(a)$ be nonnegative $\calF_{t-1}$-measurable random variables such that $\alpha_{t}(a) \le 2\gamma_t$, for all $t \in [T]$ and all $a \in \calA$.
    Let $\tilde{\ell}_t(a)$ be as in \eqref{eq:standard-ix-loss}.
    Then,
    \[
        \sum_{t=1}^T \sum_{a=1}^K \alpha_t(a) \bigl(\tilde{\ell}_t(a) - \ell_t(a)\bigr) \le \ln\lr*{1/\delta}
    \]
    holds with probability at least $1-\delta$ for any $\delta \in (0,1)$.
\end{lemma}

In our case, we require an analogous result that work when loss vectors correspond with our estimates $\tilde{\theta}_1, \ldots, \tilde{\theta}_T$.
However, these estimate have a dependency with the past actions chosen by the learner.
This requires some nontrivial changes in the proof of \citet[Lemma~1]{Neu15Implicit}.

Before that, we introduce some crucial definitions for this proof.
Let $\rho(t) = t+d_t$ be the arrival time for the realized loss $\ell_t(S_t)$ of the state $S_t$ observed at time $t \in [T]$.
Let $\tilde{\rho}(t) = t+\tilde{d}_t$ be instead the arrival time perceived by algorithm $\calB$ relative to its choice of $A_t$ at time $t$, i.e., when $\calB$ receives $\tilde{\theta}_t(S_t)$.
This also means that $\tilde{\theta}_t(S_t)$ is only defined at time $\tilde{\rho}(t) \le \rho(t)$.

Let $\pi\colon [T] \to [T]$ be the permutation of $[T]$ that orders rounds according to their value of $\tilde{\rho}$.
In other words, $\pi$ satisfies the following property:
\begin{equation} \label{eq:permut-arrival-time}
    \pi(r) < \pi(t) \quad \iff \quad \tilde{\rho}(r) < \tilde{\rho}(t) \vee \lr{\tilde{\rho}(r) = \tilde{\rho}(t) \wedge r < t} \qquad \forall r,t \in [T] \enspace.
\end{equation}
This permutation allows us to sort rounds according to the order in which \metaif feeds $\calB$ with a respective estimate for the mean loss.
In particular, the $r$-th round in this order corresponds with the round $t_r = \pi^{-1}(r)$, for any $r \in [T]$.
Hence, we can equivalently define the round $t_r$ as the round such that its estimate $\tilde{\theta}_{t_r}(S_{t_r})$ for the mean loss $\theta(S_{t_r})$ is the $r$-th estimate received by $\calB$.

Define
\begin{equation}
    \calF_{r} = \lrc*{(j, A_j, S_j, \ell_j(S_j)) \mid j \in [T], \pi(j) \le r} \qquad \forall r \in [T]
\end{equation}
as the information observed by $\calB$ by the end to the time step when we feed it the estimate relative to round $t_r$.
Note that this defines a filtration, as $\calF_{r-1} \subseteq \calF_r$ for all $r \in [T]$, which has some desirable properties thanks to the ordering $\pi$ we consider.
In particular, we have that $\tilde{d}_{t_r}, \e_{t_r}, p_{t_r}, N'_{t_r}$ are $\calF_{r-1}$-measurable random variables by the way we define them.
This property is also due to the fact that $N_{t_r}$ and $\calL'_{t_r}$ are determined when conditioning on $\calF_{r-1}$.
Moreover, we are now interested in the following importance-weighted loss estimates with implicit exploration:
\begin{equation} \label{eq:ix-loss-estimate}
    \tilde{\ell}_t(a) = \frac{\mathbbm{1}\lrs{A_t=a}}{p_t(a) + \gamma_t} \tilde{\theta}_t(s_{t}(a)) \qquad \forall t \in [T], \forall a \in \calA \enspace.
\end{equation}

\begin{corollary}
    Let $\gamma_{t_r}$ and $\alpha_{t_r}(a)$ be non-negative $\calF_{r-1}$-measurable random variables such that $\alpha_{t_r}(a) \le 2\gamma_{t_r}$, for all $r \in [T]$ and all $a \in \calA$.
    Let $\tilde{\ell}_t(a)$ be as in \eqref{eq:ix-loss-estimate}.
    Then,
    \[
        \sum_{t=1}^T \sum_{a=1}^K \alpha_t(a) \bigl(\tilde{\ell}_t(a) - \tilde{\theta}_t(s_{t}(a))\bigr) \le \ln\lr*{1/\delta}
    \]
    holds with probability at least $1-\delta$ for any $\delta \in (0,1)$.
\end{corollary}
\begin{proof}
    We follow the proof of \citet[Lemma~1]{Neu15Implicit} by considering any realization $\ell_1, \ldots, \ell_T$ of the losses.
    The main difference is that, when defining the supermartingale as in the original proof, we need to consider the terms of the sum in the order denoted by $\pi$ instead of the increasing order of $t$.
    For this reason, we rewrite the sum from the statement by following the order given by $\pi$:
    \[
        \sum_{r=1}^T \sum_{a=1}^K \alpha_{t_r}(a) \bigl(\tilde{\ell}_{t_r}(a) - \tilde{\theta}_{t_r}(s_{t_r}(a))\bigr) \enspace.
    \]
    
    At this point, we need prove that $\bbE\bigl[\tilde{\ell}_{t_r}(a) \,\big|\, \calF_{r-1}\bigr] \le \tilde\theta_{t_r}(s_{t_r}(a))$, where we recall that $t_r = \pi^{-1}(r)$.
    Also recall that $\e_{t_r}$, $p_{t_r}$ and $\gamma_{t_r}$ are $\calF_{r-1}$-measurable.
    This property allows us to prove the inequality with the conditional expectation of $\hat{\theta}_{t}$ instead of the one with the actual optimistic estimates $\Tilde{\theta}_{t}$, by the definition of the latter.
    In other words, we now need to prove that $\bbE\bigl[\hat{\ell}_{t_r}(a) \,\big|\, \calF_{r-1}\bigr] \le \hat\theta_{t_r}(s_{t_r}(a))$, where $\hat{\ell}_{t}(a) = \frac{\mathbbm{1}\lrs{A_{t}=a}}{p_{t}(a) + \gamma_{t}} \hat{\theta}_{t}(s_{t}(a))$.
    
    We can consider two cases depending on whether $\Tilde{d}_{t_r} < d_{t_r}$ is true or not (and, thus, we are in the case $\Tilde{d}_{t_r} = d_{t_r}$).
    In the first case, note that the realized losses used for computing $\hat{\theta}_{t_r}(s_{t_r}(a))$ correspond to time steps in $\calL'_{t_r}(s_{t_r}(a))$, for which there is a corresponding tuple in $\calF_{r-1}$.
    Therefore, we have that $\hat{\theta}_{t_r}(s_{t_r}(a))$ is $\calF_{r-1}$-measurable, and we can show that
    \begin{align*}
        \E{\hat{\ell}_{t_r}(a) \mathbbm{1}\lrs{\Tilde{d}_{t_r} < d_{t_r}} \;\middle|\; \calF_{r-1}}
        &= \E{\frac{\mathbbm{1}\lrs{A_{t_r} = a}}{p_{t_r}(a) + \gamma_{t_r}} \;\middle|\; \calF_{r-1}} \frac{\mathbbm{1}\lrs{\Tilde{d}_{t_r} < d_{t_r}}}{N'_{t_r}(s_{t_r}(a))} \sum_{j \in \calL'_{t_r}(s_{t_r}(a))} \ell_j(s_{t_r}(a)) \enspace.
    \end{align*}
    In the second case, we have that $\Tilde{d}_{t_r} = d_{t_r}$, which implies that $t_r \in \calL'_{t_r}(s_{t_r}(a))$ in the case $A_{t_r}=a$.
    This means that we have a corresponding tuple in $\calF_{r-1}$ only for rounds in $\calL'_{t_r}(s_{t_r}(a)) \setminus \{t_r\}$.
    Nonetheless, this does not pose an issue since we have the indicator $\mathbbm{1}\lrs{A_{t_r}=a}$, and thus $S_{t_r} = s_t(a)$.
    Indeed, we have that
    \begin{align*}
        \E{\hat{\ell}_{t_r}(a) \mathbbm{1}\lrs{\Tilde{d}_{t_r} = d_{t_r}} \;\middle|\; \calF_{r-1}}
        &= \E{\frac{\mathbbm{1}\lrs{A_{t_r} = a}}{p_{t_r}(a) + \gamma_{t_r}} \cdot \frac{\mathbbm{1}\lrs{\Tilde{d}_{t_r} = d_{t_r}}}{N'_{t_r}(s_{t_r}(a))} \sum_{j \in \calL'_{t_r}(s_{t_r}(a))} \ell_j(s_{t_r}(a)) \;\middle|\; \calF_{r-1}}  \\
        &= \E{\frac{\mathbbm{1}\lrs{A_{t_r} = a}}{p_{t_r}(a) + \gamma_{t_r}}\;\middle|\; \calF_{r-1}} \frac{\mathbbm{1}\lrs{\Tilde{d}_{t_r} = d_{t_r}}}{N'_{t_r}(s_{t_r}(a))} \sum_{\substack{j \in \calL'_{t_r}(s_{t_r}(a))\\ j \neq t_r}} \ell_j(s_{t_r}(a)) \\
        &\qquad + \E{\frac{\mathbbm{1}\lrs{A_{t_r} = a}}{p_{t_r}(a) + \gamma_{t_r}} \;\middle|\; \calF_{r-1}} \frac{\mathbbm{1}\lrs{\Tilde{d}_{t_r} = d_{t_r}}}{N'_{t_r}(s_{t_r}(a))} \ell_{t_r}(s_{t_r}(a)) \\
        &= \E{\frac{\mathbbm{1}\lrs{A_{t_r} = a}}{p_{t_r}(a) + \gamma_{t_r}}\;\middle|\; \calF_{r-1}} \frac{\mathbbm{1}\lrs{\Tilde{d}_{t_r} = d_{t_r}}}{N'_{t_r}(s_{t_r}(a))} \sum_{j \in \calL'_{t_r}(s_{t_r}(a))} \ell_j(s_{t_r}(a))
    \end{align*}
    and therefore the inequality
    \[
        \E{\hat{\ell}_{t_r}(a) \;\middle|\; \calF_{r-1}} = \E{\frac{\mathbbm{1}\lrs{A_{t_r} = a}}{p_{t_r}(a) + \gamma_{t_r}} \;\middle|\; \calF_{r-1}} \hat{\theta}_{t_r}(s_{t_r}(a))
        \le \hat{\theta}_{t_r}(s_{t_r}(a))
    \]
    is true because $\mathbbm{1}\lrs{\Tilde{d}_{t} < d_{t}} + \mathbbm{1}\lrs{\Tilde{d}_{t} = d_{t}} = 1$ for all $t \in [T]$, and by definition of $\hat{\theta}_t$.
    
    As already mentioned, this is equivalent to proving that $\bbE\bigl[\Tilde{\ell}_{t_r}(a) \,\big|\, \calF_{r-1}\bigr] \le \Tilde{\theta}_{t_r}(s_{t_r}(a))$ holds.
    By using a notation similar to the original proof, if we define $\Tilde{\lambda}_r = \sum_{a=1}^K \alpha_{t_r}(a)\Tilde{\ell}_{t_r}(a)$ and $\lambda_r = \sum_{a=1}^K \alpha_{t_r}(a)\Tilde{\theta}_{t_r}(s_{t_r}(a))$, the process $(Z_r)_{r\in[T]}$ with $Z_r = \exp\lr{\sum_{j=1}^r \lr{\Tilde{\lambda}_j - \lambda_j}}$ is a supermartingale with respect to $(\calF_r)_{r\in[T]}$ which has the same properties as in the proof of \citet[Lemma~1]{Neu15Implicit}.
    This concludes the current proof by following a similar reasoning as in the original one.
\end{proof}

Thanks to this result, we can conclude that the adoption of \dadaexp for the reduction via \metaif can guarantee a high-probability regret bound on $\hat{\calR}_T^{\calB}$ as stated in \Cref{theorem:dadaexp}, but with total delay $\tilde{\calD}_T = \sum_{t=1}^T \tilde{d}_t$ instead of $\calD_T$.

\subsection{Regret of \metaif}\label{sec:metaiferrorterms}

By \Cref{lem:concentration-unbiased-estimates}, we have that
\begin{equation} \label{eq:regret-metaif-to-dadaexp}
    R_T \le \sum_{t=1}^T \tilde{\theta}_t(S_t) - \min_{a \in \calA} \sum_{t=1}^T \tilde{\theta}_t(s_t(a)) + \sum_{t=1}^T \e_t(S_t) = \hat{\calR}_T^{\calB} + \sum_{t=1}^T \e_t(S_t)
\end{equation}
with probability at least $1-\delta/2$, where $\hat{R}_T^{\calB}$ (\Cref{eq:regret-dadaexp-estimates}) is the regret of algorithm $\calB$ when fed with $(\tilde{\theta}_t \circ s_t)_{t\in[T]}$ as losses.

\begin{lemma} \label{lem:sum-estimates-errors-bound}
    Conditioning on the event as stated in \Cref{lem:concentration-unbiased-estimates}, the sum of errors suffered from \metaif by using the loss estimates $(\tilde{\theta}_t)_{t\in[T]}$ from \Cref{eq:lowerconf-estimates,eq:unbiased-estimates} is
    \[
        \sum_{t=1}^T \e_t(S_t) \le \bigl(4+2\sqrt{2}\bigr) \sqrt{ST\ln\lr*{\frac{4ST}{\delta}}} \enspace.
    \]
\end{lemma}
\begin{proof}
    First, observe that we can rewrite the sum of errors as
    \[
        \sum_{t=1}^T \e_t(S_t) = \sum_{t=1}^T \e_t(S_t) \mathbbm{1}\lrs{\tilde{d}_t<d_t} + \sum_{t=1}^T \e_t(S_t) \mathbbm{1}\lrs{\tilde{d}_t=d_t} \enspace.
    \]
    We now provide an upper bound for the first sum of errors.
    For any $s \in \calS$, we define $\mathcal{T}_s = \lrc{t \in [T]: S_t = s}$ to be the set of all rounds when the state observed by the learner corresponds to $s$.
    We can bound it as
    \begin{align*}
        \sum_{t=1}^T \e_t(S_t)\mathbbm{1}\lrs{\tilde{d}_t<d_t} &= \sum_{s \in \calS} \sum_{t \in \calT_s} \e_t(s)\mathbbm{1}\lrs{\tilde{d}_t<d_t} \\
        &= \sqrt{2\ln\lr*{\frac{4ST}{\delta}}} \sum_{s \in \calS} \sum_{t \in \calT_s} \sqrt{\frac{1}{N'_t(s)}}\mathbbm{1}\lrs{\tilde{d}_t<d_t} \\
        &\le 2\sqrt{\ln\lr*{\frac{4ST}{\delta}}} \sum_{s \in \calS} \sum_{t \in \calT_s} \sqrt{\frac{1}{M_t(s)}}\mathbbm{1}\lrs{\tilde{d}_t<d_t} \tag{because $N'_t(s) \ge \frac12 M_t(s)$}\\
        &\le 4\sqrt{\ln\lr*{\frac{4ST}{\delta}}} \sum_{s \in \calS} \sqrt{M_T(s)} \tag{since $M_t(s)$ is increasing over $\calT_s$}\\
        &\le 4\sqrt{ST\ln\lr*{\frac{4ST}{\delta}}} \enspace,
    \end{align*}
    where the second inequality holds because $N'_t(S_t) = N_t(S_t) \ge \frac12 M_t(S_t)$ when $\tilde{d}_t < d_t$ since $M_t(S_t) \le N_t(S_t) + \sigma_t$, while the last one follows by Jensen's inequality and the fact that $\sum_{s\in\calS} M_T(s) = T$.

    As a last step, we provide an upper bound to the second sum.
    Let $J_s = \lrc{r \in \calT_s : \tilde{d}_r=d_r}$ and notice that $|J_s| \le |\calT_s| = M_T(s)$.
    Observe that $\rho(t) = \tilde{\rho}(t)$ for each round $t$ such that $\tilde{d}_t=d_t$, and thus by \Cref{eq:permut-arrival-time} we have that
    \[
        \pi(r) < \pi(t) \iff \rho(r) < \rho(t) \vee \lr{\rho(r) = \rho(t) \wedge r < t}
    \]
    for all $r,t \in [T]$ such that $\tilde{d}_r=d_r$ and $\tilde{d}_t=d_t$.
    Define $\nu_s : J_s \to \big[\abs{J_s}\big]$ by
    \[
        \nu_s(t) = \abs*{\lrc{r \in J_s \,:\, \pi(r) \le \pi(t)}} \qquad \forall t \in J_s \enspace.
    \]
    Observe that $\nu_s(t) \le N'_t(s) = \abs{\calL'_t(s)}$ for all $s \in \calS$ and all $t \in J_s$.
    This is due to the fact that $\nu_s(t)$ counts a subset of $\calL'_t(s)$; to be precise, we have that $\nu_s(t) = \abs{\calL'_t(s) \cap J_s}$.
    Moreover, notice that the condition $\pi(r) \le \pi(t)$ defines a total order over $J_s$.
    Hence, $\nu_s(t)$ counts the number of elements of $J_s$ preceding $t \in J_s$ (including $t$ itself) in this total order.
    This implies that $\nu_s$ is a bijection between $J_s$ and $\bigl[|J_s|\bigr]$.
    Then, using a similar reasoning as before, we show that
    \begin{align*}
        \sum_{t=1}^T \e_t(S_t) \mathbbm{1}\lrs{\tilde{d}_t=d_t}
        &= \sqrt{2\ln\lr*{\frac{4ST}{\delta}}} \sum_{s \in \calS} \sum_{t \in \calT_s} \sqrt{\frac{1}{N'_t(s)}}\mathbbm{1}\lrs{\tilde{d}_t=d_t} \\
        &= \sqrt{2\ln\lr*{\frac{4ST}{\delta}}} \sum_{s \in \calS} \sum_{t \in J_s} \sqrt{\frac{1}{N'_t(s)}} \tag{by definition of $J_s$}\\
        &\le \sqrt{2\ln\lr*{\frac{4ST}{\delta}}} \sum_{s \in \calS} \sum_{t \in J_s} \sqrt{\frac{1}{\nu_s(t)}} \tag{since $\nu_s(t) \le N'_t(s)$ for $t \in J_s$}\\
        &\le 2\sqrt{2\ln\lr*{\frac{4ST}{\delta}}} \sum_{s \in \calS} \sqrt{\abs{J_s}} \tag{since $\nu_s(t)$ is bijective}\\
        &\le 2\sqrt{2\ln\lr*{\frac{4ST}{\delta}}} \sum_{s \in \calS} \sqrt{M_T(s)} \tag{since $\abs{J_s} \le M_T(s)$}\\
        &\le 2\sqrt{2ST\ln\lr*{\frac{4ST}{\delta}}} \enspace. \tag{by Jensen's inequality}
    \end{align*}
\end{proof}

\upperboundmetaif*
\begin{proof}[Proof of \Cref{theorem:metaif}]
    By \Cref{eq:regret-metaif-to-dadaexp}, the regret $R_T$ can be bounded as
    \[
       R_T \le \hat{\calR}_T^{\calB} + \sum_{t=1}^T \e_{t}(S_t) \le \hat{\calR}_T^{\calB} + 7\sqrt{ST\ln\frac{4ST}{\delta}}
    \]
    with probability at least $1-\delta/2$, where the last inequality follows by \Cref{lem:sum-estimates-errors-bound}.
    From what we argued in \Cref{app:reductiontodada}, we can upper bound $\hat{\calR}_T^{\calB}$ using the high-probability regret bound of \dadaexp. 
    Notice that the delays incurred by \dadaexp via \metaif are those given when providing the estimates $\lr{\tilde{\theta}_t}_{t\in[T]}$.
    We denote these delays by $\tilde{d}_1, \ldots, \tilde{d}_T$, and the total delay perceived by \dadaexp is thus $\tilde{\calD}_T = \sum_{t=1}^T \tilde{d}_t$.
    Hence, from the improved bound for \dadaexp in \Cref{eq:improveddada-1}, we have that
    \[
        \hat\calR_T^{\calB} \leq 2\sqrt{2KT \lr*{3\ln(K) + \ln\lr{4/\delta}}} + 2\sqrt{\tilde\calD_T \lr*{3\ln(K) + \ln\lr{4/\delta}}} + \frac{\sigma_{\max} + 2}{2}\ln(4/\delta)
    \]
    holds with probability at least $1-\delta/2$.
    The combination of the above two inequalities, together with \Cref{lem:total-delay-bound}, concludes the proof.
\end{proof}

\subsection{Regret of \metaifswitch}\label{sec:metaifswitch-analysis}

\upperboundmetaifswitch*
\begin{proof}[Proof of \Cref{theorem:metaifswitch}]
    Let $t^* \in [T]$ be the last round before \metaifswitch switches from \dadaexp to \metaif, i.e., the last round that satisfies $\frakD_{t^*} C_{K,4\delta} \leq 49 ST \ln \frac{8ST}{\delta}$.
    Then, define $a^* \in \argmin_{a} \sum_{t=1}^T \theta(s_t(a))$.
    We may decompose regret as  
    \begin{align*}
        R_T &= \sum_{t = 1}^{t^*} \Bigl(\theta(S_t) - \theta(s_t(a^*))\Bigr) + \sum_{t = t^*+1}^{T} \Bigl(\theta(S_t) - \theta(s_t(a^*))\Bigr) \\
        &\le \underbrace{\sum_{t = 1}^{t^*} \theta(S_t) - \min_{a \in \calA} \sum_{t=1}^{t^*} \theta(s_t(a))}_{R_{t^*}} + \underbrace{\sum_{t = t^*+1}^T \theta(S_t) - \min_{a \in \calA} \sum_{t = t^*+1}^T \theta(s_t(a))}_{R_{t^*:T}} \enspace.
    \end{align*}
    The incurred delay until time $t^*$ is $\frakD_{t^*}$.
    Thus, from \Cref{eq:improveddada}, we get that the following bound
    \begin{align}
        R_{t^*} \leq 2\sqrt{2K t^* C_{K,2\delta}} + \sqrt{2 t^* \ln\frac{12}{\delta}} + 2\sqrt{\frakD_{t^*} C_{K,2\delta}} + \frac{\sigma_{\max}+2}{2}\ln\frac{6}{\delta}\label{eq:beforeswitch}
    \end{align}
    holds with probability at least $1-\delta/2$, where we recall that $C_{K,\delta} = 3\ln K + \ln(12/\delta)$.
    If our algorithm never switches, then $t^* = T$ and we get the bound in \eqref{eq:beforeswitch} for $R_T$. 
    Note that this is no greater than the upper bound in the statement as $\sqrt{\frakD_T C_{K,2\delta}} \le 7\sqrt{ST \ln(8ST/\delta)}$ by definition of $t^*$ in this case.
    
    Otherwise, we use the switching condition $\sqrt{\frakD_{t^*}C_{K,2\delta}} \leq 7\sqrt{ST \ln(8ST/\delta)}$ along with the fact that $\sqrt{t^*\ln(12/\delta)} \leq \sqrt{Kt^* C_{K,2\delta}}$ to get
    \begin{equation}
        R_{t^*} \leq 3\sqrt{2K t^* C_{K,2\delta}} + 14\sqrt{ST \ln\frac{8ST}{\delta}} + \frac{\sigma_{\max}+2}{2}\ln\frac{6}{\delta} \enspace. \label{eq:regret-T0}
    \end{equation}
    Furthermore, \cref{theorem:metaif} directly gives us an upper bound for $R_{t^*:T}$ since \metaifswitch runs \metaif for $t > t^*$ with the confidence parameter set to $\delta/2$.
    We just need to bound the total incurred delays of these rounds, namely $\tilde{\calD}_{t^*:T}$.
    Let $\sigma_{t}'$ be the outstanding observations for any round $t > t^*$ as perceived by the execution of \metaif starting after round $t^*$, that is, when considering only delays $(d_t)_{t > t^*}$.
    It is immediate to observe that $\sigma_t' \leq \sigma_{t}$ and thus $\max_{t > t^*}\sigma_{t}' \leq \max_{t > t^*}\sigma_{t}$.
    Moreover, from \Cref{lem:total-delay-bound} we have
    \[
        \tilde{\calD}_{t^*:T} \leq \calD_{\Phi'} \enspace,
    \]
    where $\Phi'$ denotes a set of $\min\lrc{T-t^*, 2S \sigma_{\max}'}$ rounds with the largest delays among $(d_t)_{t > t^*}$, with $\sigma_{\max}' = \max_{t > t^*}\sigma_t'$.
    So we have
    \[
        \calD_{\Phi'} \leq \calD_{\Phi}
    \]
    due to the fact that $\abs{\Phi'} = \min\lrc{T-t^*, 2S \sigma_{\max}'} \leq \min\lrc{T, 2S \sigma_{\max}} = \abs{\Phi}$.
    Therefore, from \Cref{theorem:metaif} we obtain
    \begin{equation}\label{eq:regret-T0:T}
        R_{t^*:T} \leq 2\sqrt{2K(T-t^*)C_{K,3\delta}} + 7\sqrt{ST \ln\frac{8ST}{\delta}} + 2\sqrt{\calD_{\Phi} C_{K,3\delta}} + \frac{\sigma_{\max}+2}{2}\ln\frac{8}{\delta}
    \end{equation}
    with probability at least $1-\delta/2$.
    We conclude the proof by combining \Cref{eq:regret-T0,eq:regret-T0:T} along with the fact that $\sqrt{t^*} + \sqrt{T-t^*} \leq \sqrt{2T}$ to get that the bound
    \[
    R_T \leq 6\sqrt{KTC_{K,2\delta}} + 3\min\lrc*{7\sqrt{ST \ln\frac{8ST}{\delta}}, \sqrt{\calD_T C_{K,2\delta}}} + 2\sqrt{\calD_{\Phi} C_{K,2\delta}} + (\sigma_{\max}+2)\ln\frac{8}{\delta}
    \]
    holds with probability at least $1-\delta$.
\end{proof}

\subsection{Expected Regret Analysis of \metaifswitch with \tsallisinf}\label{app:expectedregret}

\expectedupperbound*
\begin{proof}[Proof of \Cref{theorem:expectedregret}]
    We begin by studying of expected regret of \metaif and we then give a regret analysis of \metaifswitch.
    When running \metaif, we use the unbiased empirical mean estimators $(\hat\theta_t)_{t \in [T]}$ as the mean loss estimates, rather than the lower confidence bounds $(\tilde\theta_t)_{t \in [T]}$.
    The expected regret is defined as
    \[
        \bbE[R_T] = \sum_{t=1}^T \bbE\lrs{\theta(S_t)} - \sum_{t=1}^T \theta(s_t(a^*))\enspace,
    \]
    where $a^* = \min_{a \in \calA} \sum_{t = 1}^T \theta(s_t(a))$. 
    Here we use a version of \tsallisinf that is tailored for the delayed bandits problem \citep{zimmert20}, which guarantees a bound in expectation on the regret
    \[
        \hat{\calR}^{\mathrm{Tsallis}}_T(a) = \sum_{t=1}^T \hat{\theta}_t(S_t) -  \sum_{t=1}^T \hat{\theta}_t(s_t(a))
    \]
    against any fixed action $a \in \calA$, using the loss estimates $\lrc{\hat{\theta}_t}_{t \in [T]}$.
    Observe that this regret is defined in terms of our estimates, as required in our case.
    By \citet[Theorem~1]{zimmert20}, \tsallisinf guarantees that its expected regret is
    \begin{align*}
        \E{\hat{\calR}_T^{\mathrm{Tsallis}}(a^*)} = \bbE\lrs*{\sum_{t=1}^T \hat{\theta}_t(S_t) -  \sum_{t=1}^T \hat{\theta}_t(s_t(a^*))}
        \leq 4\sqrt{KT} + \sqrt{8\tilde\calD_T \ln K}
        \leq 4\sqrt{KT} + \sqrt{8\calD_{\Phi} \ln K} \enspace,
    \end{align*}
    where the last inequality uses \Cref{lem:total-delay-bound}.
    Then, we can focus on our notion of regret and use the above regret bound to obtain that
    \begin{align}
        \bbE[R_T] &= \bbE\lrs*{R_T - \hat{\calR}_T^{\mathrm{Tsallis}}(a^*)} + \E{\hat{\calR}_T^{\mathrm{Tsallis}}(a^*)} \notag\\
        &= \bbE\Biggl[\sum_{t=1}^T \bigl(\theta(S_t) - \hat\theta_t(S_t)\bigr)\Biggr] + \bbE\lrs*{\sum_{t=1}^T \bigl(\hat\theta_t(s_t(a^*)) - \theta(s_t(a^*))\bigr)} + \E{\hat{\calR}_T^{\mathrm{Tsallis}}(a^*)} \notag\\
        &\leq \bbE\Biggl[\underbrace{\sum_{t=1}^T \bigl(\theta(S_t) - \hat\theta_t(S_t)\bigr)}_{\Delta}\Biggr] + \bbE\lrs*{\sum_{t=1}^T \bigl(\hat\theta_t(s_t(a^*)) - \theta(s_t(a^*))\bigr)} + 4\sqrt{KT} + \sqrt{8\calD_{\Phi} \ln K} \enspace. \label{eq:regretdecomp}
    \end{align}
    
    We know that our mean estimator is unbiased.
    Therefore, we have that $\bbE\lrs{\hat\theta_t(s_t(a^*))} = \theta(s_t(a^*))$ for any $t \in [T]$, meaning that the second term in the right-hand side of \eqref{eq:regretdecomp} is equal to zero.

    On the other hand, we can apply \Cref{lem:concentration-unbiased-estimates} to get the following bound for $\Delta$ that holds with probability at least $1-\delta/2$ for any $\delta \in (0,1)$:
    \begin{equation} \label{eq:hatepsilonbound}
        \Delta \leq \min\lrc*{\frac12 \sum_{t=1}^T \e_t(S_t), T} \enspace,
    \end{equation}
    where we recall that $\e_t(s) =  \sqrt{\frac{2}{N'_t(s)} \ln\frac{4ST}{\delta}}$.
    In particular, the inequality $\Delta \le T$ is true in general.
    By \Cref{lem:sum-estimates-errors-bound}, we can bound the right-hand side of \eqref{eq:hatepsilonbound} as
    \[
        \frac12\sum_{t=1}^T \e_t(S_t) \leq \frac{7}{2} \sqrt{ST \ln\frac{4ST}{\delta}}
    \]
    when conditioning on the event as in the statement of \Cref{lem:concentration-unbiased-estimates}. If we denote such an event as $\calE$, we have that $\pr{\overline{\calE}} \le \delta/2$ and that $\E{\Delta \mid \calE} \le \frac{7}{2}\sqrt{ST\ln\lr{4ST/\delta}}$.
    As a consequence, we notice that
    \[
        \E{\Delta} = \E{\Delta \mid \calE}\pr{\calE} + \E{\Delta \mid \overline{\calE}}\pr{\overline{\calE}}
        \le \frac{7}{2}\sqrt{ST \ln\frac{4ST}{\delta}} + \frac{\delta}{2}T \le 5\sqrt{ST \ln\lr{2ST}} + 1
    \]
    where in the last inequality we set $\delta = 2/T$.
    Since we assume that $S \ge 2$, we can easily observe that $\E{\Delta} \le 6\sqrt{ST\ln\lr{2ST}}$.
    Plugging this into \Cref{eq:regretdecomp} gives us
    \begin{equation}\label{eq:expectedregret-metaif}
        \E{R_T} \leq 4\sqrt{KT} + \sqrt{8\calD_{\Phi} \ln K} + 6\sqrt{ST \ln(2ST)} \enspace.
    \end{equation}
    
    At this point, we can proceed to the proof of the overall bound on the expected regret of \metaifswitch.
    The behaviour of \metaifswitch follows the same principle as before, but the switching condition is different:
    \[
        \sqrt{8\frakD_t \ln K} > 6\sqrt{ST\ln(2ST)}\enspace.
    \]
    Similar to the analysis of \metaifswitch in \Cref{sec:metaifswitch-analysis}, we decompose the regret into
    \[
        \bbE[R_T] \leq \underbrace{\sum_{t = 1}^{t^*}\bbE\lrs{\theta(S_t)} - \min_{a \in \calA} \sum_{t=1}^{t^*} \theta(s_t(a))}_{R_{t^*}} + \underbrace{\sum_{t = t^*+1}^T \bbE\lrs{\theta(S_t)} - \min_{a \in \calA} \sum_{t = t^*+1}^T \theta(s_t(a))}_{R_{t^*:T}} \enspace,
    \]
    where $t^*$ is the last round satisfying $\sqrt{8\frakD_{t^*}} \le 6\sqrt{ST\ln\lr{2ST}}$.
    Then, we have 
    \begin{equation}\label{eq:expected-beforeswitch}
        \bbE[R_{t^*}] \leq 4\sqrt{Kt^*} + \sqrt{8\frakD_{t^*}\ln K}\enspace.
    \end{equation}
    If $t^* = T$ then $R_{t^*} = R_T$ and we get the bound in \eqref{eq:expected-beforeswitch}, where we note that $\sqrt{8\frakD_T \ln K} \le 6\sqrt{ST\ln\lr{2ST}}$ by definition of $t^*$ in this case, and we can replace $\frakD_{T}$ by $\calD_T$.
    Otherwise, $t^* < T$ and we can apply the bound for \metaif from \eqref{eq:expectedregret-metaif}, along with the fact that the total incurred delay after round $t^*$ is upper bounded by $\calD_{\Phi}$, in order to derive an upper bound for $\bbE[R_{t^*:T}]$ that is
    \begin{equation}\label{eq:expected-afterswitch}
        \bbE[R_{t^*:T}] \leq 4\sqrt{K(T-t^*)} + \sqrt{8\calD_{\Phi} \ln K} + 6\sqrt{ST \ln(2ST)} \enspace.
    \end{equation}
    Finally, if we use the fact that $\sqrt{8\frakD_{t^*}} \leq 6\sqrt{ST\ln(2ST)}$ (by definition of $t^*$) in \eqref{eq:expected-beforeswitch}, and combine it with \eqref{eq:expected-afterswitch}, we conclude that
    \[
        \bbE[R_T] \leq 4\sqrt{2KT} + \sqrt{8\calD_{\Phi} \ln K} + 2\min\lrc*{6\sqrt{ST \ln(2ST)}, \sqrt{8\calD_T \ln K}} \enspace,
    \]
    where we also used the fact that $\sqrt{t^*} + \sqrt{T-t^*} \le \sqrt{2T}$.
\end{proof}

\section{Proofs for the Lower Bounds}\label{app:lowerbounds}

\rethmstateslbwithdependence*
\begin{proof}[Proof of \Cref{th:STlowerB}]
    Assume without loss of generality that $K = 2$ and let $\mathcal{S} = \lrc{h_1, \dots, h_S}$ be the finite set of possible states.
    Let $S' = \floor{\min\{S/2, d\}}$ and let $I_1, \dots, I_T$ be the actions chosen by the considered algorithm.
    Split the $T$ time steps into $m = \floor{T/S'}$ blocks  $B_1, \dots, B_m$ of equal size $S'$, eventually leaving $\le S'-1$ extra time steps.
    We assume with no loss of generality that the last step corresponds to the end of the $m$-th block.
    The feedback formed by the losses of the actions chosen by the algorithm in a certain block is received only after the last time step of the same block since $S \le 2d$.  
    Define $b_i = (i-1)S'+1$ for all $i \in [m]$.
    We assume that the learner receives \emph{all} the realized losses $\ell_t(s_t(A))$ for all $t \in B_i$ and all $A \in \{1,2\}$ at the end of each block, which means that we are in a full information setting, as this only helps the algorithm.

    Now, we define a specific sequence of assignments from actions to states, and construct losses so that the expected regret becomes sufficiently large.
    Let $s_t(A) = h_{2(t-b_i)+A}$ for all $t \in B_i$, all $i \in [m]$ and all $A \in \lrc{1,2}$; this means that, for the first time step of any block, actions 1 and 2 will be assigned to states $h_1$ and $h_2$ respectively, then to $h_3$ and $h_4$ respectively in the next time step of the same block, and so on.
    Let $\e = \frac14\sqrt{\frac{S'}{2T\ln\lr{4/3}}} \in \lrs{0,\frac14}$ and let $\theta^{(A)} \in \mathbb{R}^2$ be a vector of mean losses such that $\theta^{(A)}_i = \frac12 - \mathbb{I}\lrc{i=A}\e$, for each $A \in \lrc{1,2}$.
    We simplify the notation with $\bbE_A\lrs{\cdot} = \E{\cdot \,\middle|\, \theta^{(A)}}$ and $\bbP_A\lr{\cdot} = \pr{\cdot \,\middle|\, \theta^{(A)}}$, where the conditioning on $\theta^{(A)}$ means that we sample losses for each state assigned to $i \in \lrc{1,2}$ such that they are Bernoulli random variables with mean $\theta^{(A)}_i$.
    In particular, conditioning on $\theta^{(A)}$, we sample independent Bernoulli random variables $X^i_1, \dots, X^i_m$ with mean $\theta^{(A)}_i$, one for each block, for $i \in \lrc{1,2}$.
    Then, the losses are defined as $\ell_t(s_t(i)) = X^i_{j}$ for each $t \in B_j$ and each $j \in [m]$.

    We can now proceed to show a lower bound for the expected pseudo-regret.
    Let $T_i$ be the number of times the learner chooses action $i$ over all $T$ time steps.
    The expected pseudo-regret over the two instances determined by $\theta^{(k)}$ for $k \in \lrc{1,2}$ adds up to
    \[
        \bbE_1\lrs{R_T} + \bbE_2\lrs{R_T} = \e\lr{2T - \bbE_1\lrs{T_1} - \bbE_2\lrs{T_2}} \enspace.
    \]
    Following the standard analysis, we show that the difference $\bbE_2\lrs{T_2} - \bbE_1\lrs{T_2}$ is such that
    \[
        \bbE_2\lrs{T_2} - \bbE_1\lrs{T_2} \le T\cdot \dtv(\bbP_2, \bbP_1) \le T\sqrt{\frac12 \KL{\bbP_1 \,\|\, \bbP_2}} \enspace,
    \]
    where the last step follows by Pinsker's inequality.

    Let $\lambda_i = \lrc{(I_t, \ell_t(S_t(1)), \ell_t(S_t(2))) \mid t \in B_i}$ be the feedback set known to the learner by the end of block $B_i$, and let $\lambda^i = (\lambda_1, \dots, \lambda_i)$ be the tuple of all feedback sets up to the end of block $B_i$.
    Denote by $\bbP_{k,i}\lr{\cdot}$ the probability measure of feedback tuples $\lambda^i$ conditioned on $\theta^{(A)}$.
    By the chain rule for the relative entropy, we can observe that
    \begin{align*}
        \KL{\bbP_1 \,\|\, \bbP_2} &= \sum_{i=1}^m \sum_{\lambda^{i-1}} \bbP_1\lr{\lambda^{i-1}} \KL{\bbP_{1,i}\lr{\cdot \mid \lambda^{i-1}} \|\, \bbP_{2,i}\lr{\cdot \mid \lambda^{i-1}}} \\
        &\le \sum_{i=1}^m \sum_{\lambda^{i-1}} \bbP_1\lr{\lambda^{i-1}} 16\e^2 \ln\lr{4/3} \\
        &= 16m\e^2 \ln\lr{4/3} \enspace,
    \end{align*}
    where we used the fact that each relative entropy $\KL{\bbP_{1,i}\lr{\cdot \mid \lambda^{i-1}} \,\|\, \bbP_{2,i}\lr{\cdot \mid \lambda^{i-1}}}$ corresponds to the sum of the relative entropy between two Bernoulli distributions with means $1/2$ and $1/2-\e$ and that between Bernoulli distributions with means $1/2-\e$ and $1/2$, respectively, which is upper bounded by $16\e^2\ln\lr{4/3}$ for $\e \in [0, 1/4]$.
    This follows by an application of the chain rule for the relative entropy, as well as from the fact that the distribution of $I_t$ is the same under both $\bbP_{1,i}\lr{\cdot \mid \lambda^{i-1}}$ and $\bbP_{2,i}\lr{\cdot \mid \lambda^{i-1}}$, for all $t \in B_i$ and any $\lambda^{i-1}$.
    Therefore, we have that
    \begin{align*}
        \bbE_2\lrs{T_2} - \bbE_1\lrs{T_2} \le 2\e T\sqrt{2m\ln\lr{4/3}}
    \end{align*}
    which also implies that
    \[
        \bbE_1\lrs{R_T} + \bbE_2\lrs{R_T} \ge \e T\lr*{1 - 2\e\sqrt{2\frac{T}{S'}\ln\lr{4/3}}}
        = \frac{\e T}2 \ge \frac18\sqrt{\frac{\floor{S/2}T}{2\ln(4/3)}} \ge \frac18\sqrt{\frac{ST}{6\ln\lr{4/3}}} \enspace,
    \]
    where we used the facts that $m \le T/S'$ and that $\floor{S/2} \ge S/3$ for any integer $S \ge 2$.
    This means that the expected pseudo-regret of the learner has to be $\frac1{16}\sqrt{\frac{ST}{6\ln\lr{4/3}}}$ at least in one of the two instances. Now, for $S > 2 d$ we use the same construction, but now we only use $2d$ states, which leads to the promised $\Omega(\sqrt{\min\{S, d\}T})$ lower bound. 
\end{proof}

\rethmlbfixeddelay*
\begin{proof}[Proof of \Cref{thm:lb-fixed-delay}]
    Let $S' = \min\lrc{\floor{\frac{S}{2}}, \floor{\frac{T}{d+1}}} \ge 1$. 
    We consider the first $(d+1)S'$ rounds of the game and divide them into $S'$ blocks $B_1, \ldots, B_{S'}$ of same length $d+1$.
    In this way, we ensure that the feedback for any time step in some block is revealed to the learner only after its final round.

    Without loss of generality, we can assume that the learner observes all the losses of one block immediately after its last time step; this only helps the learner since they would observe only the incurred losses at possibly later rounds otherwise.
    We can further simplify the problem by assuming that losses are deterministic functions of the states, i.e., $\ell_t \equiv \theta$ for every round $t$.
    This also means that the problem turns into an easier, full-information version of our problem with deterministic losses.
    Now, let the adversary choose the action-state mappings such that for each block index $i$ and each action $a \in \mathcal{A}$,
    $S_{t}(a) = S_{t'}(a) \in \lrc{s_{2i-1}, s_{2i}}$ for all $t,t' \in B_i$.
    Furthermore, we assume that the losses are chosen such that $\theta(s_{2i-1}) \in \{0, 1\}$ and $\theta(s_{2i}) = 1 - \theta(s_{2i-1})$ for all $i \in [S']$.
    In this construction, the learner cannot obtain any useful information from the states of a block because of the delays.
    Moreover, the states observed in one block are not observed again in the other blocks.
    
    It thus suffices to prove a lower bound for a standard full information game with $S'$ rounds and loss range $[0, d+1]$.
    Hence, we can conclude that the expected regret of any algorithm has to be
    \[
        \E{R_T} = \Omega\lr*{(d+1)\sqrt{S'}} = \Omega\lr*{\min\lrc*{(d+1)\sqrt{S}, \sqrt{(d+1)T}}} \enspace.
    \]
\end{proof}

\section{Action-State Mappings and Loss Means Used in the Experiments}
\label{app:mappings}
\Cref{table:stochastic_action_state_map} and \Cref{table:adversarial_action_state_map} describe the instances used to generate the data for the experiments of \Cref{sec:experiments}.

\begin{table}[h!]
\centering
\begin{tabular}{|c|c|c|c|}
\hline
Mean loss & $s=1$ & $s=2$ & $s=3$ \\
 \hline
$\theta(s)$  & 0.2    &0.4&   0.8\\
 \hline
\hline
Mapping & $P(1|a)$ & $P(2|a)$ & $P(3|a)$\\
\hline
    $a=1$ & 0.8 & 0.1 & 0.1 \\
    $a=2$ & 0.4 & 0.5 & 0.1 \\
    $a=3$ & 0.3 & 0.7 & 0.0 \\
    $a=4$ & 0.5 & 0.3 & 0.2 \\
\hline
\end{tabular}
\caption{Mean losses and stochastic action-state mapping for Experiment~1 in \Cref{sec:experiments}.}
\label{table:stochastic_action_state_map}
\end{table}

\begin{table}[h!]
\centering
    \begin{tabular}{|c|c|c|c|}
    \hline
    Mean loss & $s=1$ & $s=2$ & $s=3$ \\
     \hline
$\theta(s)$ & 0    &1&   1\\
     \hline
    \multicolumn{4}{c}{Environment 1}\\
    \hline
Mapping & $P(1|a)$ & $P(2|a)$ & $P(3|a)$\\
\hline
        $a=1$ & 0.06 & 0.47 & 0.47 \\
        $a=2$ & 0 & 0.50 & 0.50 \\
        $a=3$ & 0 & 0.50 & 0.50 \\
        $a=4$ & 0 & 0.50 & 0.50 \\
    \hline
    \multicolumn{4}{c}{Environment 2}\\
    \hline
Mapping & $P(1|a)$ & $P(2|a)$ & $P(3|a)$\\
\hline
        $a=1$ & 1 & 0 & 0 \\
        $a=2$ & 0.94 & 0.03 & 0.03 \\
        $a=3$ & 0.94 & 0.03 & 0.03 \\
        $a=4$ & 0.94 & 0.03 & 0.03 \\
\hline    
    \end{tabular}
    \caption{Mean losses and stochastic action-state mappings for Experiment~2 in \Cref{sec:experiments}.}
    \label{table:adversarial_action_state_map}
\end{table}


\end{document}